\documentclass[11pt]{article}

\usepackage[final]{acl}

\usepackage{times}
\usepackage{latexsym}
\usepackage[T1]{fontenc}
\usepackage[utf8]{inputenc}
\usepackage{microtype}
\usepackage{inconsolata}
\usepackage{mathtools}
\usepackage{amsmath, amssymb, amsthm}
\usepackage{graphicx}
\graphicspath{{./}{./figures/}{./figures/no_intervention/}{./figures/intervention1/}{./figures/intervention2/}{./figures/intervention3/}{./figures/intervention4/}{./figures/intervention5/}{./figures/intervention5_2/}{./figures/obs1_appendix/}{./figures/obs2_appendix/}}
\usepackage{cleveref}
\usepackage{enumitem} 
\usepackage{pgfplots}
\pgfplotsset{compat=1.17}
\usepackage{caption}
\usepackage{subcaption} 
\usepackage{float}
\usepackage{xcolor}
\usepackage{bbm}
\newif\ifYRMcomments
\newif\ifBacklogcomments
\newif\ifResolvedcomments

\Backlogcommentsfalse

\YRMcommentstrue 
\Resolvedcommentstrue

\newtheorem{theorem}{Theorem}
\newtheorem{lemma}{Lemma}

  \title{Attention Sinks Are Provably Necessary in Softmax Transformers:\\ Evidence from Trigger-Conditional Tasks}

  \author{Yuval Ran-Milo\\
    Tel Aviv University\\
    \texttt{yuvalmilo@mail.tau.ac.il}
  }

  \date{}

\begin{document}
  \maketitle
  \thispagestyle{plain} 

  \begin{abstract}
      Transformers often display an \emph{attention sink}: probability mass concentrates on a fixed, content-agnostic position. 
      Are sinks a byproduct of the optimization/training regime? Or are they sometimes functionally necessary in softmax Transformers? We prove that, in some settings, it is the latter: computing a simple trigger-conditional behavior \emph{necessarily} induces a sink in softmax self-attention models. 
      Our results formalize a familiar intuition: normalization over a probability simplex must force attention to collapse onto a stable anchor to realize a default state (e.g., when the model needs to ignore the input).
      We instantiate this with a concrete task: when a designated trigger token appears, the model must return the \emph{average of all preceding token representations}, and otherwise output zero, a task which mirrors the functionality of attention heads in the wild \citep{barbero2025llmsattendtoken,Guo2024ActiveDormantAH}. 
      We also prove that non-normalized ReLU attention can solve the same task without any sink, confirming that the normalization constraint is the fundamental driver of sink behavior. 
      Experiments validate our predictions and demonstrate they extend beyond the theoretically analyzed setting: softmax models develop strong sinks while ReLU attention eliminates them in both single-head and multi-head variants.
    \end{abstract}

  \section{Introduction}\label{sec:intro}
  Transformers \cite{vaswani2023attentionneed} frequently concentrate attention on an early position in a way that is largely insensitive to content.
  This \emph{attention sink} has been reported for small and large models alike \citep{xiao2023efficient,gu2025when,Guo2024ActiveDormantAH}.
  It occurs under a variety of positional schemes—absolute/learned embeddings, ALiBi, RoPE, and even without explicit positional encodings \citep{press2021train,su2021roformer,gu2025when}—and similar behavior shows up in multimodal and vision settings, as well as in diffusion language models \citep{Kang2025See,wang2025mirage,Feng2025EDIT, rulli2025dlmsinks}.
  The breadth of contexts points to a pervasive pattern, not a peculiarity of any single model or training regime.

  This pattern has significant practical consequences.
  When probability mass concentrates on a fixed position, attention can be diverted away from other tokens and downstream accuracy can be affected \citep{Yu2024Unveiling}.
  Sinks can also worsen numerical issues relevant to compression and quantization \citep{sun2024massive,lin2024duquant,bondarenko2023quantizable,son-etal-2024-prefixing}, distort attention-based interpretability analyses \citep{Guo2024ActiveDormantAH}, and complicate streaming and long-context inference \citep{xiao2023efficient}.
  Analogous sink effects have also been documented in vision and multimodal settings, where they waste representational capacity on irrelevant visual tokens \citep{Kang2025See,wang2025mirage,Feng2025EDIT}.  (See Appendix \ref{sec:appendix-sink-impact} for an extended discussion on the practical motivations for mitigating attention sinks.)
  
  Why is sink behavior so common?
  One plausible account is an \emph{inductive bias}—a phenomenon documented in other settings \cite{soudry2024implicitbiasgradientdescent,arora2019implicitregularizationdeepmatrix, ranmilo2026outcomebasedrlprovablyleads}—
  whereby the learning setup (model class and optimization procedure) steers solutions toward models that exhibit attention sinks, even when sink-free alternatives exist.
  In this work we argue that, in certain settings, this isn't the case, and sink behavior is \emph{functionally essential}: all models that successfully compute a natural class of functions must exhibit sinks.\footnote{We do not claim sinks are unavoidable in all architectures (e.g., sinks do not appear in gated attention or Mamba-based models \cite{qiu2025gatedattentionlargelanguage, endy2025mambaknockoutunravelingfactual}). Rather, we prove they are a necessary consequence of softmax attention.}

    We investigate this claim theoretically by introducing a \emph{trigger-conditional task}: a model must output the mean of past tokens at a designated trigger position, and output zero (a no-op) everywhere else.
    This formulation captures the core mechanism of empirically observed attention heads ``in the wild'' \citep{barbero2025llmsattendtoken,Guo2024ActiveDormantAH} which aggregate context when triggered and use a sink to remain dormant otherwise (see \cref{sec:related} for more details).
    We prove that attention sinks are necessary for softmax attention to solve this task.
  Specifically, we consider a synthetic, trigger-conditional task on sequences in which each token representation consists of: (i) a \emph{BOS indicator} equal to one only for the first token; (ii) a \emph{trigger indicator} equal to one only at the trigger position; (iii) a \emph{non-trigger non-BOS indicator} equal to one for all remaining tokens; and (iv) i.i.d.\ samples from a continuous distribution in the content coordinates.
  The target is intuitive: the model writes nothing to the residual stream at every position (i.e., outputs the zero vector), except at the unique trigger position where it should write the \emph{mean of all preceding non-\texttt{BOS}~\footnote{We exclude \texttt{BOS} from the average because it contains no 
  input-dependent content.} token vectors}.

  Our main results are necessity theorems for \emph{softmax} self-attention: for single-layer models (\cref{thm:main}), any model that achieves vanishing error on this task must place attention arbitrarily close to~$1$ (the maximal possible value) on a \emph{fixed sink token} (the \texttt{BOS} token) at \emph{all} non-trigger positions; for multi-layer models (\cref{thm:multilayer}), we show that at least one layer must exhibit sink behavior at some non-trigger position~\footnote{\label{fn:multilayer-sinks}Indeed, we empirically see in \cref{sec:experiments} that sinks do form, but not in all positions and layers (see \cref{fig:no-sink-head}).}.
  At a high level, we formalize a widely held intuition: normalization of attention scores forces the model to concentrate probability mass on a stable anchor whenever it needs to produce a default output, independent of the variable input content.
  We complement these necessity theorems with a constructive result (\cref{thm:relu}): ReLU attention can solve the same task with zero attention on the \texttt{BOS} token, demonstrating that the normalization constraint is the primary driver of sink formation.

  Experiments on both single-layer and multi-layer models provide supporting evidence (\cref{sec:experiments}).
  Single-layer softmax Transformers trained on the task develop attention sinks with near-unit mass on \texttt{BOS} when no trigger is present, aligning with our theoretical analysis.
  Swapping softmax for ReLU attention eliminates sink formation while preserving task accuracy, confirming that the softmax normalization constraint---rather than the task structure or optimization dynamics---is the fundamental driver of the sink behavior.
  We observe these patterns across both single-layer and deeper multi-head multi-layer architectures, demonstrating that our theoretical insights capture fundamental properties of normalization-based attention mechanisms.

  Overall, our contributions are as follows:
  \begin{enumerate}[itemsep=0.15em]
    \item We introduce a trigger-conditional task that models the mechanism of attention heads observed ``in the wild'' \citep{barbero2025llmsattendtoken,Guo2024ActiveDormantAH} (\cref{sec:task-def}). 
    \item We prove that any single-layer softmax attention model achieving vanishing error on this task must place nearly all attention on a fixed sink token (the \texttt{BOS} token) at \emph{every} non-trigger position (\cref{thm:main}).
    \item We extend this to multi-layer models, showing that at least one layer must place nearly all attention on the \texttt{BOS} token at \emph{some} non-trigger position\textsuperscript{\ref{fn:multilayer-sinks}} (\cref{thm:multilayer}).
    \item We show that the softmax normalization is the driver of sink formation by showing the existence of a ReLU attention model that perfectly solves the same task without any sink formation (\cref{thm:relu}).
  \end{enumerate}

  \section{Sinks Empirically Enable No-Op Behaviors in Real Models}\label{sec:related}
  \begin{figure}[!t]
    \centering
    \includegraphics[width=0.67\columnwidth]{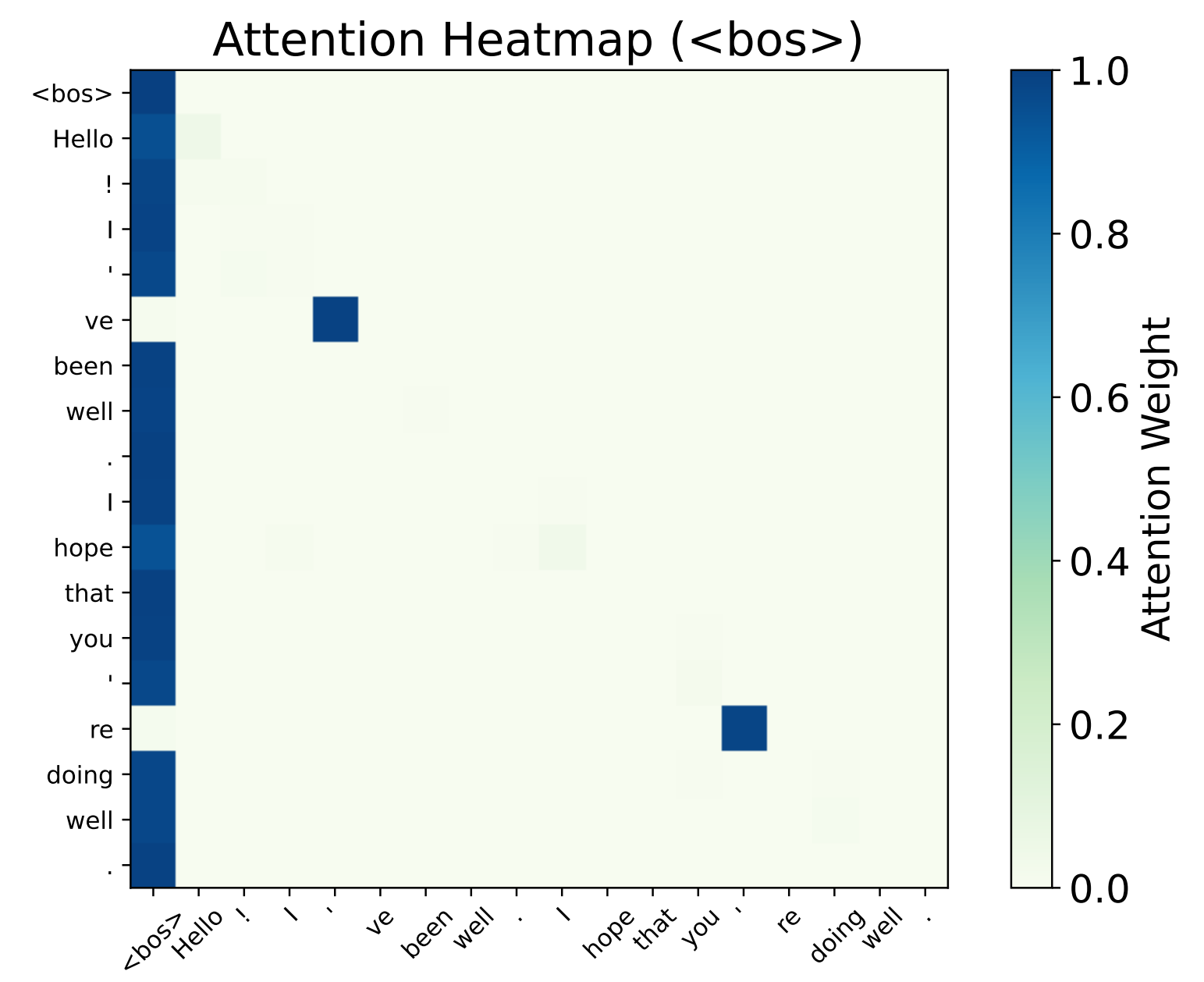}
      \caption{Reproduced from \citet{barbero2025llmsattendtoken}\protect\footnotemark: an attention head that fires on an apostrophe trigger and otherwise attends to \texttt{BOS}.}
    \label{fig:apostrophe}
  \end{figure}
  \footnotetext{\label{fn:apost}Licensed under Creative Commons Attribution 4.0 (CC BY 4.0). Minor cropping for layout; no other changes. License: \url{https://creativecommons.org/licenses/by/4.0/}.}
  
    In realistic empirical settings, attention sinks frequently appear in attention heads implementing a no-op behavior in the absence of specific triggers.
  \citet{barbero2025llmsattendtoken} demonstrate this directly: their case study of an ``apostrophe head'' in Gemma~7B shows two operating modes—firing on apostrophe triggers and otherwise attending to \texttt{BOS} as a default no-operation (no-op) (\cref{fig:apostrophe})\textsuperscript{\ref{fn:apost}}.
  Similarly, \citet{Guo2024ActiveDormantAH} document an active–dormant head in Llama~2–7B that switches between active computation on code-like inputs and dormant sink behavior on text-like inputs (\cref{fig:git})\textsuperscript{\ref{fn:git}}.
  \begin{figure}[!t]
    \centering
    \includegraphics[width=1\columnwidth]{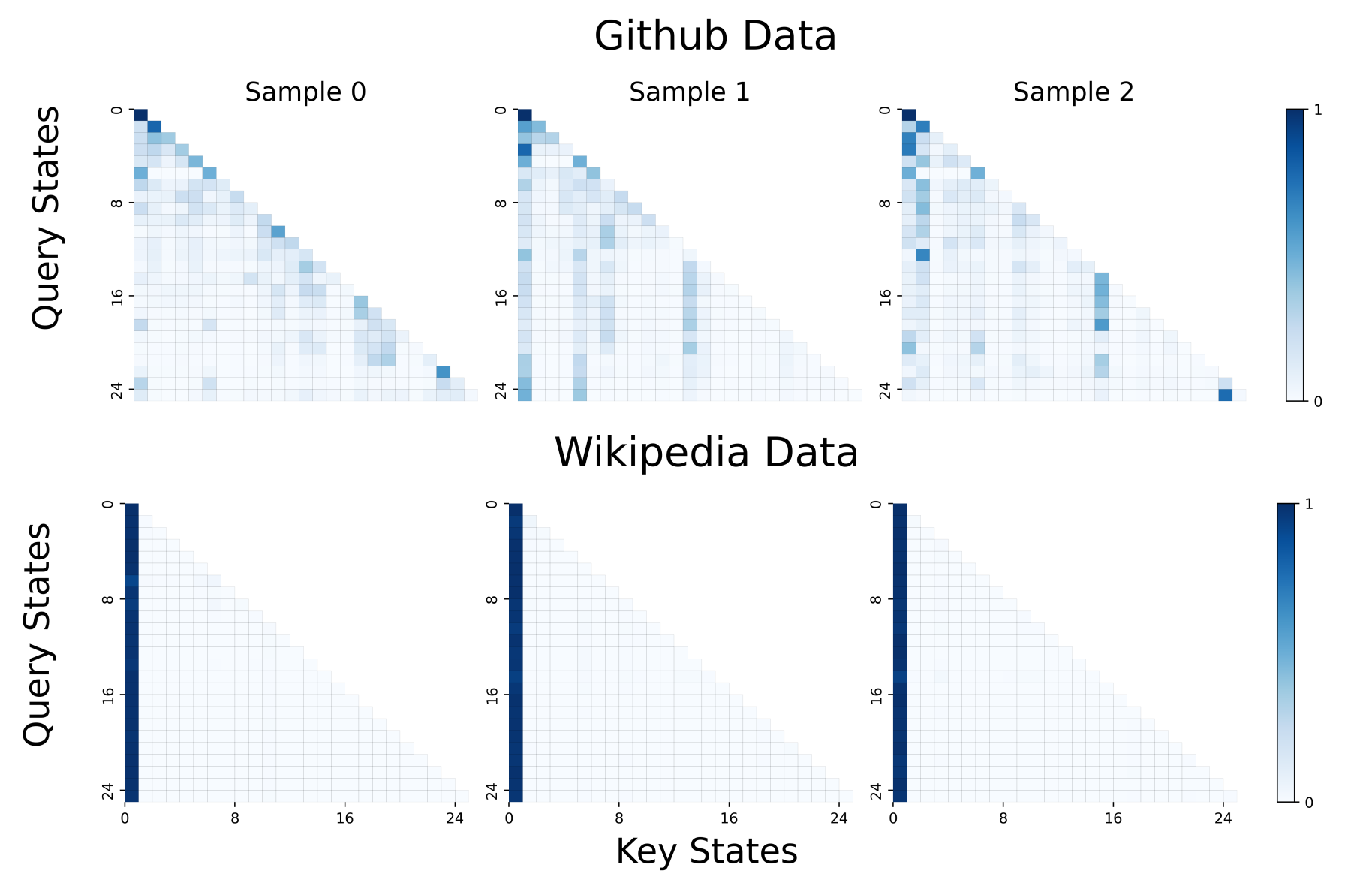}
      \caption{Reproduced from \citet{Guo2024ActiveDormantAH}\protect\footnotemark: an active--dormant attention head in Llama~2--7B. On code-like inputs (GitHub, top), the head exhibits diverse attention patterns; on text-like inputs (Wikipedia, bottom), it collapses to an attention sink on position~0.}
    \label{fig:git}
  \end{figure}
  \footnotetext{\label{fn:git}Reproduced with written permission of the authors from \url{https://github.com/GuoTianYu2000/Active-Dormant-Attention}.}
  Notably, \citet{Guo2024ActiveDormantAH} report that sink behavior diminishes under certain non-softmax/activation variants; in particular, replacing softmax with ReLU attention eliminates sinks, consistent with our theoretical result (\cref{thm:relu}).
  
    These works complement our theoretical perspective.
  \citet{barbero2025llmsattendtoken} argue that sinks enable controlled information mixing, with \texttt{BOS} serving as a stable anchor.
  \citet{Guo2024ActiveDormantAH} analyze the training dynamics behind sink formation—how these patterns emerge during optimization.
  In contrast, our work establishes a theoretical necessity of sink behavior in softmax attention and its absence in ReLU attention via expressiveness analyses \emph{regardless of optimization and training schemes}.
  We include illustrative figures from \citet{barbero2025llmsattendtoken} (\cref{fig:apostrophe})\textsuperscript{\ref{fn:apost}} and \citet{Guo2024ActiveDormantAH} (\cref{fig:git})\textsuperscript{\ref{fn:git}} to highlight that our synthetic task captures key aspects of real sink behavior—sinks emerge to implement a no-op when no trigger fires. 
  See Appendix~\ref{sec:related_extended} for more related works.
  
  \section{Theory and Results}\label{sec:theory}
  We now set up our analysis.
  We introduce the task in \cref{sec:task-def}, explain why this task is meaningful and how its assumptions match realistic modeling in \cref{sec:task-motivation}, introduce the model architectures in \cref{sec:attention-weights}, and state our main necessity claims in \cref{sec:main-result}.
  
  \subsection{Notation and Setup}
  We write $\mathbb{R}_{>0}$ for the positive reals and $\mathbb{N}_{\geq k}$ for the natural numbers at least $k$.
  We use $\mathbbm{1}\{\cdot\}$ for the indicator function and denote $[k]=\{1,\dots,k\}$. Let $n\in\mathbb{N}_{\geq 5}$ be the input dimension and $L \in \mathbb{N}_{\geq 4}$ denote the sequence length. We write sequences as $\mathbf{x}=(\mathbf{x}^{(1)},\ldots,\mathbf{x}^{(L)})^\top \in \mathbb{R}^{L \times n}$ with tokens vectors $\mathbf{x}^{(i)}\in\mathbb{R}^{n \times 1}$.
  
  \subsection{Task Definition}\label{sec:task-def}
    We define a synthetic task designed to capture the mechanism of attention sinks ``in the wild''.
  Empirical studies show that attention heads in LLMs frequently implement trigger-conditional behavior: they aggregate context upon detecting a specific trigger, and attend to a sink token to effectively ``switch off'' otherwise \citep{barbero2025llmsattendtoken,Guo2024ActiveDormantAH} (see \cref{sec:related} for more details).
  Our task isolates this structure: the model must detect a trigger token and, \emph{only at the trigger position}, write to the residual stream the mean of prior content, and write the zero vector at all other positions.

  \begin{figure*}[!t]
    \centering
    \includegraphics[width=1\textwidth]{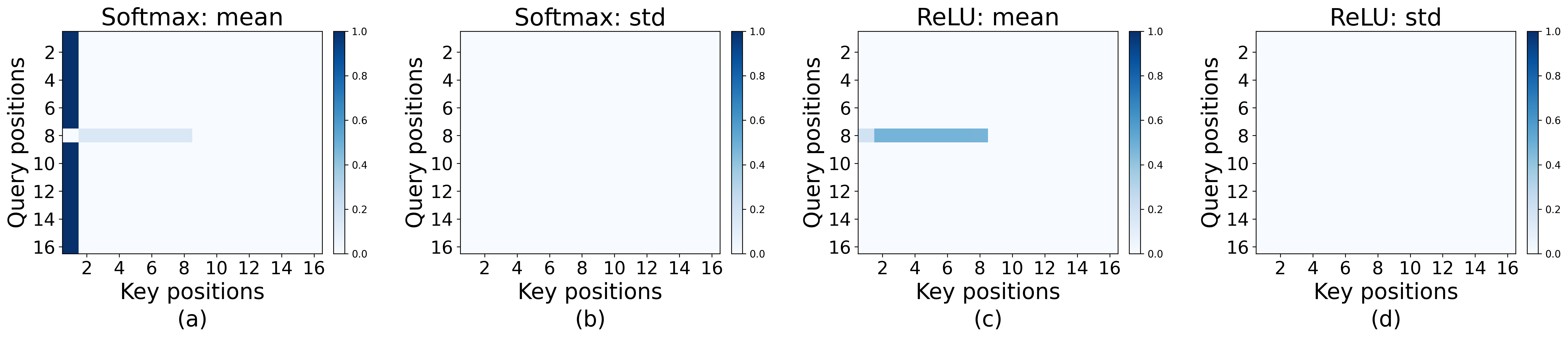}
    \caption{\textbf{Validation of \cref{thm:main} and \cref{thm:relu}.} (a) Mean attention weights for softmax attention across 1000 test examples with trigger at position 8. Dark regions indicate high attention mass concentrated on \texttt{BOS} (position 1) at non-trigger positions. (b) Standard deviation of softmax attention weights shows negligible variance, confirming stable sink behavior. (c) Mean attention weights for ReLU attention show no sink formation—attention on \texttt{BOS} remains near zero. (d) Standard deviation for ReLU attention confirms consistent behavior across examples.}
    \label{fig:experiments}
  \end{figure*}

  \subsubsection{Input Distribution} 
  Input tokens lie in $\mathbb{R}^n$ (for some $n \in \mathbb{N}_{\geq 5}$) and consist of four coordinate types: 
  (i) a \emph{\texttt{BOS} indicator} (coordinate~1), equal to one only for the first token; 
  (ii) a \emph{trigger indicator} (coordinate~2), equal to one only for the trigger token; 
  (iii) a \emph{non-trigger non-\texttt{BOS} indicator} (coordinate~3), equal to one for all remaining tokens; and 
  (iv) \emph{content coordinates} ($4 \leq k \leq n$), drawn i.i.d.\ from some continuous distribution  (except for the \texttt{BOS} token, for which the content coordinates are fixed to zero, as this token contains no input-dependent content). 
  
  Formally, we construct our input distribution $\mathcal{D}$ as follows. We sample a \emph{trigger position} $j$ uniformly from $\{2,\ldots,L\}$, and construct a sequence as follows:
  \begin{itemize}[leftmargin=*]
  \item \textbf{Position 1 (\texttt{BOS}):} Coordinate~1 is one; all other coordinates are zero.
  \item \textbf{Position $j$ (Trigger token):} Coordinate~2 is one; coordinates $4\le k\le n$ are i.i.d.\ from some continuous distribution. All other coordinates are zero.
  \item \textbf{Positions $i\neq 1,j$:} Coordinate~3 is one; coordinates $4\le k\le n$ are i.i.d.\ from some continuous distribution. All other coordinates are zero.
  \end{itemize}

  \subsubsection{Target Output} 
  The target output $\mathbf{y}^{(i)}$ is the zero vector $\mathbf{0}$ at all positions except the trigger position $i=j$, where it equals $(j{-}1)^{-1}\sum_{k=2}^{j}\mathbf{x}^{(k)}$, the mean of all preceding non-\texttt{BOS} tokens (including the trigger itself).
  
  \subsubsection{Loss Function}
  Following standard approximation-theoretic formulations in neural-network expressivity theory \cite{alberti2023sumformeruniversalapproximationefficient,kidger2020universalapproximationdeepnarrow}, we evaluate hypotheses using the following worst-case loss:
  $\mathcal{L}(f) \,=\, \sup_{(\mathbf{x},\mathbf{y})\in \text{support}(\mathcal{D})} \max_{i \in [L]} \big\|\mathbf{y}^{(i)} - f(\mathbf{x})^{(i)}\big\|_2$.

  \subsection{Task Motivation and Justification}\label{sec:task-motivation}
    This setup captures a basic and pervasive pattern in sequence modeling: \emph{aggregate context upon a trigger, otherwise perform a no-op}
   \citep{barbero2025llmsattendtoken,Guo2024ActiveDormantAH} (see \cref{sec:related} for more details).
  Our task distills this to its minimal form: detect a trigger and compute the mean of prior content, or, otherwise, output zero.~\footnote{Our analysis applies almost as-is to a broader class of trigger-conditional problems, such as key-query retrieval where a query must extract a specific previous token (e.g., marked by a feature bit) while ignoring others, resembling the apostrophe head in \cref{fig:apostrophe}\textsuperscript{\ref{fn:apost}}. We analyze the averaging task for clarity, leaving the formal characterization of the full class of tasks necessitating sinks to future work.}
  
  The design choices we make are less arbitrary than they may appear.
  Many aspects are without loss of generality: the BOS indicator, the trigger indicator, and the non-trigger non-BOS indicator channels can be any three mutually orthogonal vectors via a change of basis; we fix them to coordinates 1, 2, and 3 for simplicity.
  While having such fixed indicator channels feels somewhat arbitrary, it is a natural way to model position-type information that an MLP layer can easily learn to inject into the residual stream in practice (e.g., by writing a constant vector).

  \begin{figure*}[!t]
    \centering
    \includegraphics[width=0.97\textwidth]{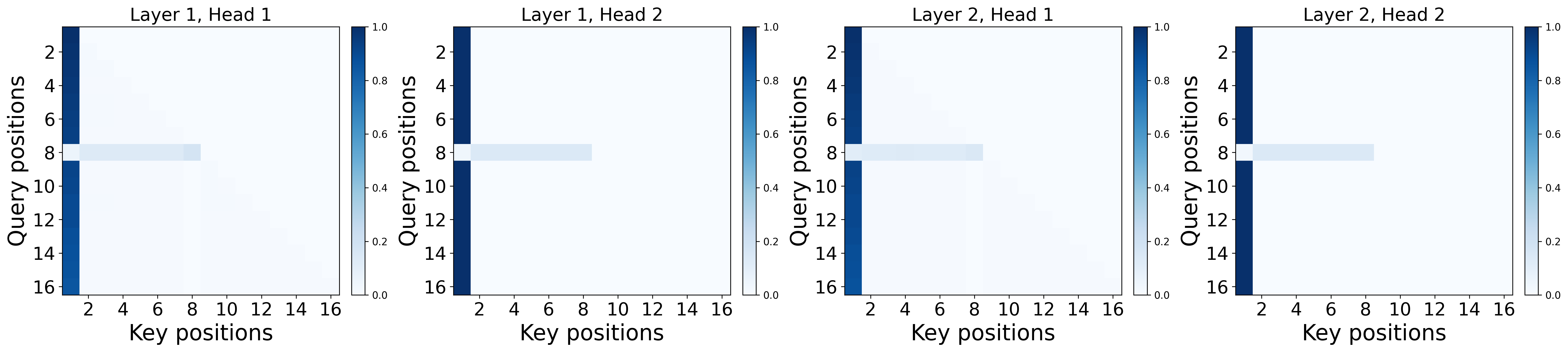}
    \caption{\textbf{Multi-layer multi-head validation.} Attention patterns for a 2-layer 2-head softmax model on a random input (with trigger at position 8). All heads exhibit strong sink behavior.}
    \label{fig:experiments-multilayer}
  \end{figure*}

  \subsection{Model Architecture} \label{sec:attention-weights}
  We study self-attention models with two variants of attention mechanisms.
  We denote the learnable parameter of a single-layer attention model by $\mathbf{W}_Q, \mathbf{W}_K, \mathbf{W}_V, \mathbf{W}_O \in \mathbb{R}^{n \times n}$ for queries, keys, values, and output projection respectively.
  For input sequence $\mathbf{x} = (\mathbf{x}^{(1)}, \ldots, \mathbf{x}^{(L)})^\top \in \mathbb{R}^{L \times n}$, we calculate the attention weights $\alpha_{i,j}$ as defined below for each attention variant (softmax or ReLU). The model output is then computed as $f(\mathbf{x})^{(i)} = \mathbf{W}_O \sum_{j=1}^i \alpha_{i,j} \mathbf{W}_V \mathbf{x}^{(j)}$.

  \paragraph{Softmax Attention.} The \emph{attention weight} from position $i$ to position $j \leq i$ is given by:
  \begin{align*}
\alpha_{i,j} = \frac{\exp(\mathbf{x}^{(i)} \mathbf{W}_Q \mathbf{W}_K^\top (\mathbf{x}^{(j)})^\top)}{\sum_{k=1}^i \exp(\mathbf{x}^{(i)} \mathbf{W}_Q \mathbf{W}_K^\top (\mathbf{x}^{(k)})^\top)}
  \end{align*}

  \paragraph{ReLU Attention.} For ReLU attention, we replace the softmax normalization with element-wise ReLU. We divide the scores by the number of positions up to the current position $i$, excluding the BOS token~\footnote{
    This scaling is necessary because ReLU attention cannot naturally compute averages: concatenating the input sequence to itself would double the output at the final position while keeping the average the same. 
  Alternatively, we could have defined the task to \emph{sum} over past tokens for the ReLU model instead of \emph{averaging}, which would yield an analogous theorem without requiring scaling.
   Moreover, a similar scaling \emph{would not work for softmax attention}, as our analysis would hold for any such variant.
  }. Namely, if we define $n_{i} = \max\{i-1,\,1\}$, then we have 
  $\alpha_{i,j} = \mathrm{ReLU}(\mathbf{x}^{(i)} \mathbf{W}_Q \mathbf{W}_K^\top (\mathbf{x}^{(j)})^\top)/n_{i}.$

\paragraph{Multi-Layer Attention.} A $D$-layer softmax/ReLU model is the composition $f = f^{(D)} \circ \cdots \circ f^{(1)}$, where each $f^{(d)}$ is a single-layer softmax/ReLU attention model. We denote by $\alpha_{i,j}^{(d)}$ the attention weight at position $i$ attending to position $j$ in layer $d$.

\subsection{Main Result}\label{sec:main-result}
  We are now ready to state our theoretical results.
  Our central contribution is threefold: 
  (i) we establish that an attention sink is \emph{necessary} at \emph{every} non-trigger position for single-layer softmax attention to solve the trigger-conditional task (\cref{thm:main});
  (ii) we prove that in multi-layer softmax attention, at least one position must exhibit sink behavior (\cref{thm:multilayer}); 
  \footnote{Indeed, we empirically see in \cref{sec:exp-multilayer} (e.g., \cref{fig:no-sink-head}) that sinks do form, but not in all positions and layers.}
  and (iii) we prove constructively that ReLU attention can solve the same task \emph{without} any sink behavior (\cref{thm:relu}).
  This contrast directly demonstrates that the softmax normalization constraint—not the task structure or optimization dynamics—is the fundamental driver of attention sinks. 
  
\begin{theorem}[Single-Layer Attention Sink Necessity]\label{thm:main}
For any $\varepsilon, \delta \in \mathbb{R}_{>0}, L \in \mathbb{N}_{\geq 4}, n \in \mathbb{N}_{\geq 5}$, and a bounded probability density function $\mathcal{P}$, there exists a constant $\eta \in \mathbb{R}_{>0}$ such that the following holds. Consider any single-layer softmax attention\footnote{\label{fn:generalization}Our analysis immediately extends to any attention mechanism whose weights $\alpha_{i,j}$ satisfy: (i)~\emph{normalization}---$\sum_{j \leq i} \alpha_{i,j} \geq c$ for some constant $c > 0$; and (ii)~\emph{monotonicity}---inserting an additional key into positions $1,\ldots,i$ does not increase $\alpha_{i,j}$ for any existing key~$j$.} model $f$ with loss $\mathcal{L}(f) \leq \eta$  on sequences with length $L$ and dimension $n$ where content coordinates are drawn from $\mathcal{P}$.\footnote{It is easy to show that such an $f$ exists for any $\eta \in \mathbb{R}_{>0}$.} Then with probability at least $1 - \delta$, for all non-trigger positions $i \neq j$, we have $\alpha_{i,1} \geq 1 - \varepsilon$.
\end{theorem}


\begin{figure*}[!t]
  \centering
  \includegraphics[width=\textwidth]{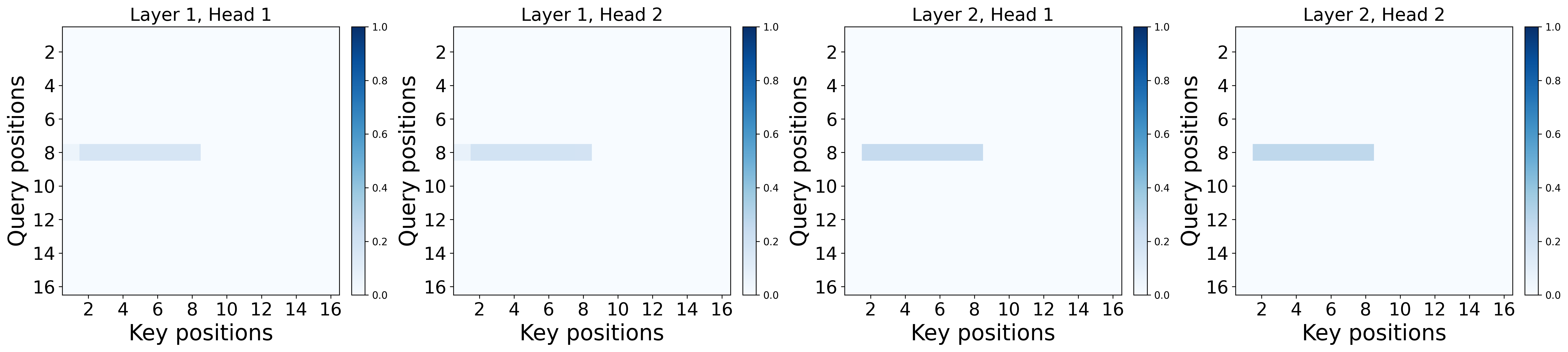}
  \caption{\textbf{ReLU attention: 2-layer 2-head model.} Attention patterns on a single test input (trigger at position 8). No sink formation occurs in any head; attention on \texttt{BOS} remains near zero throughout.}
  \label{fig:relu-2h2d}
\end{figure*}

\begin{proof}[Proof sketch (full proof in \cref{proof:main})]
  Suppose for contradiction that $\alpha_{i,1}\le 1-\varepsilon$ at some non-trigger position $i$ with probability at least $\delta>0$, even as $\eta:=\mathcal{L}(f)\to 0$. 
  On this event a constant amount of attention mass falls on non-\texttt{BOS} tokens; by pigeonhole there exist indices $i_0,h_0$ and a constant $\gamma>0$ such that $\alpha_{i_0,h_0}\ge\gamma$ on a positive-measure set.

  Since every non-trigger position must output $\mathbf{0}$ with error at most $\eta$, and adding more keys can only decrease any fixed softmax weight, one can reduce to short prefixes and show that whenever $h\le i$ are both non-trigger positions, $\|\alpha_{i,h}\mathbf{V}\mathbf{x}^{(h)}\|_2\le O(\eta)$. On the positive-measure set where $\alpha_{i,h}\ge\gamma$, this gives $\|\mathbf{V}\mathbf{x}^{(h)}\|_2= O(\eta/\gamma)$: the value map must crush a positive-probability set of non-trigger tokens.

  By bounded density and independence of the content coordinates, for every content coordinate $m\ge 4$ this crushed set contains two tokens $\mathbf{z},\mathbf{z}'$ that agree on all coordinates except $m$, where they differ by at least a constant. 
  Transplant them into two sequences with trigger at position $3$: $(\texttt{BOS},\mathbf{z},\mathbf{t})$ and $(\texttt{BOS},\mathbf{z}',\mathbf{t})$. The targets at the trigger position differ by $\tfrac{1}{2}(\mathbf{z}-\mathbf{z}')$, which has a $\Omega(1)$ component along $e_m$. The prediction at position $3$ is $\hat{\mathbf{y}}^{(3)}(\mathbf{z})=\alpha_{3,1}\mathbf{V}e_1+\alpha_{3,2}\mathbf{V}\mathbf{z}+\alpha_{3,3}\mathbf{V}\mathbf{t}$; the first two terms are $O(\eta)$ by the crushing bound, and the third lies in the span of the fixed vector $\mathbf{v}:=\mathbf{V}\mathbf{t}$. Projecting onto $\mathbf{v}^\perp$ removes the trigger contribution entirely, so the two projected predictions are $O(\eta)$-close, while the projected targets remain $\Omega(1)$-apart (choosing $m$ so $e_m$ has a nontrivial component in $\mathbf{v}^\perp$). This contradicts $\eta\to 0$.
  \end{proof}

\begin{theorem}[Multi-Layer Attention Sink Necessity]\label{thm:multilayer}
For any $\varepsilon, \delta \in \mathbb{R}_{>0}, L \in \mathbb{N}_{\geq 4}, n \in \mathbb{N}_{\geq 5}$ and a bounded probability density function $\mathcal{P}$, there exists a constant $\eta \in \mathbb{R}_{>0}$ such that the following holds. Consider any $D$-layer softmax attention\textsuperscript{\ref{fn:generalization}} model $f$ with loss $\mathcal{L}(f) \leq \eta$ on sequences with length $L$ and dimension $n$ where content coordinates are drawn from $\mathcal{P}$.\footnote{It is easy to show that such an $f$ exists for any $\eta \in \mathbb{R}_{>0}$.} Then over all inputs with trigger position $j \geq 3$, with probability at least $1 - \delta$, there exists at least one layer $d \in \{1,\ldots,D\}$ and a non-\texttt{BOS} non-trigger position $i \neq j$ such that  $\alpha_{i,1}^{(d)} \geq 1 - \varepsilon$.
\end{theorem}
\vspace{1em}

\begin{proof}[Proof sketch (full proof in \cref{proof:multilayer})]
We unroll the multi-layer network and apply similar reasoning as in \cref{thm:main}: if no layer exhibits sink behavior, the effective attention weights on content tokens remain large, forcing the value map to crush them to zero, which again contradicts the sensitivity required at the trigger position.
\end{proof}

  \begin{theorem}[ReLU Attention Without Sinks]\label{thm:relu}
For any $L \in \mathbb{N}_{\geq 4}$ and $n \in \mathbb{N}_{\geq 3}$, there exists a one-layer ReLU attention model $f$ with loss $\mathcal{L}(f) = 0 $ such that for any input sequence $\mathbf{x}$ with trigger position $j$, and any non-trigger position $i \neq j$ we have $\alpha_{i,1} = 0$.
\end{theorem}

\begin{proof}[Proof sketch (full proof in \cref{proof:relu})]
We provide a simple explicit construction. By choosing query and key weights to align with the trigger indicator coordinate and non-trigger non-BOS indicator coordinate, we ensure that attention scores are equal to some positive constant at the trigger position (where they compute the average) and zero otherwise. Since ReLU does not enforce normalization, the model can output the zero vector by simply having zero attention weights, without needing a sink.
\end{proof}

\section{Experiments}\label{sec:experiments}

 We validate our theoretical predictions on the synthetic trigger-conditional task.
  In \cref{sec:exp-theory}, we train single-layer single-head models to validate \cref{thm:main} and \cref{thm:relu}.
  In \cref{sec:exp-multilayer}, we validate our multi-layer findings (\cref{thm:multilayer}) in more realistic settings by training multi-layer multi-head models with residual connections.
  All experiments use sequences of length $L=16$; training details are in Appendix~\ref{sec:appendix-training}. Code for reproducing our experiments is available at \url{https://github.com/YuvMilo/sinks-are-provably-necessary}.

  \subsection{Single-Layer Models}\label{sec:exp-theory}
  We first validate \cref{thm:main} and \cref{thm:relu} on single-layer single-head models.

  \paragraph{Experiment 1: Softmax Attention Forms Sinks.} 
  \Cref{thm:main} predicts that softmax attention models achieving low loss must have a strong attention sink at all non-trigger positions.
  To test this, we visualize the mean and standard deviation of attention weights across 1000 test examples with trigger position $j=8$ (\cref{fig:experiments}, panels a and b).
  The model places near-unit attention mass on position 1 at every non-trigger position, with negligible variance across examples.

  \paragraph{Experiment 2: ReLU Attention Avoids Sinks.}
  \Cref{thm:relu} establishes that ReLU attention can solve the same task with zero attention on \texttt{BOS}.
  We replace softmax with ReLU attention while keeping all other parameters identical (\cref{fig:experiments}, panels c and d).
  The ReLU model achieves comparable task accuracy without developing sink behavior: attention weights on position 1 remain near zero throughout the sequence. This observation reinforces that sinks are not a byproduct of the task or training dynamics, but a direct consequence of the normalization geometry.

  \subsection{Multi-Layer Multi-Head Models}\label{sec:exp-multilayer}
  \Cref{fig:experiments-multilayer} shows attention patterns for a 2-layer 2-head softmax model: all heads exhibit strong sink behavior across non-trigger positions.
  In deeper models, sinks appear in \emph{some but not all} heads, consistent with \cref{thm:multilayer}, which guarantees existence rather than ubiquity.
  \Cref{fig:no-sink-head} shows an example: in a 4-layer 4-head softmax model that achieves low loss, head~3 in layer~4 places near-zero attention on \texttt{BOS}, while other heads in the same network develop clear sinks (\cref{fig:appendix-softmax-4h4d} in Appendix~\ref{sec:appendix-additional-exp}).
  Finally, replacing softmax with ReLU attention eliminates sink formation entirely in multi-layer models as well: \cref{fig:relu-2h2d} shows that no head of a 2-layer 2-head ReLU model develops a sink, and the same holds for a 4-layer 4-head ReLU model (see \cref{fig:appendix-relu-4h4d} in Appendix~\ref{sec:appendix-additional-exp}).

\section{Conclusions and Practical Implications}\label{sec:implications}

Our results show that for trigger-conditional behaviors, attention sinks are not an optimization artifact but a structural necessity: when a model must maintain a stable default (no-op) output on typical inputs while performing a content-dependent computation upon a recognizable trigger, softmax normalization forces sink formation.
This has direct practical consequences: it can help practitioners distinguish between mitigation strategies that are fundamentally limited and those that address the root cause.

Specifically, sink-removal interventions operating \emph{within} the softmax mechanism may be inherently limited for such computations.
Penalizing BOS attention, spreading attention mass, or post-hoc reweighting may degrade the no-op guarantee, or cause the model to recreate an equivalent anchor elsewhere (a different position, head, or layer).
In this sense, our results provide a principled reason to expect that simply ``fighting'' sinks without relaxing the simplex constraint can be counterproductive for trigger-conditional circuits: the sink may be the very mechanism that makes the circuit possible.

At the same time, the contrast with ReLU attention (\cref{thm:relu}) clarifies a more promising direction.
If sinks are undesirable for a downstream goal—e.g., they waste representational capacity \citep{Yu2024Unveiling}, confound attention-based analyses \citep{Guo2024ActiveDormantAH}, or create quantization-unfriendly outliers \citep{sun2024massive}—the right lever is to change how ``off'' states are represented, via non-normalized attention, explicit gating, or other mechanisms that can output zero without allocating probability mass.

\begin{figure}[!t]
  \centering
  \includegraphics[width=0.65\columnwidth]{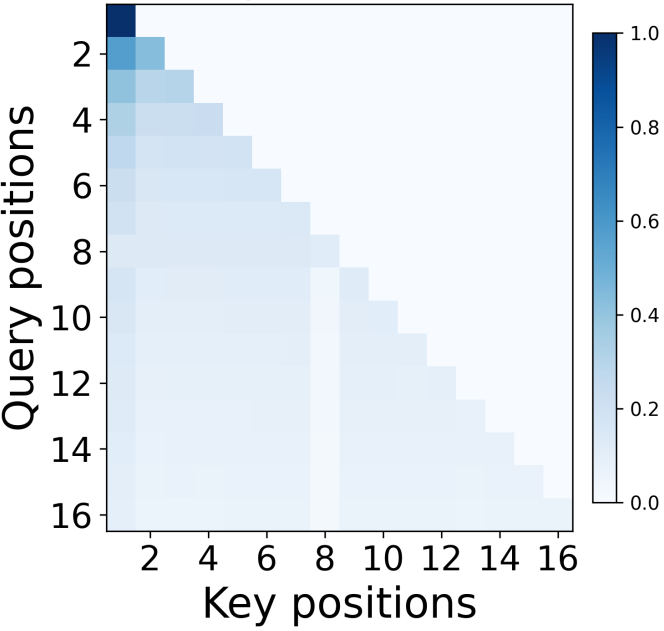}
  \caption{\textbf{A softmax head without a sink in multilayer Transformer.} Attention pattern of head~3, layer~4 in a 4-layer 4-head softmax model that achieves low loss. This head places near-zero attention on \texttt{BOS} at all non-\texttt{BOS} positions, while other heads in the same network exhibit strong sinks (see \cref{fig:appendix-softmax-4h4d} for all heads). This confirms the existential nature of \cref{thm:multilayer}: a sink must exist \emph{somewhere} in the network, but not in every head.}
  \label{fig:no-sink-head}
\end{figure}

More broadly, we hope our results can help guide future work on designing sink-free attention mechanisms that directly support no-op operations.

\section{Limitations}\label{sec:limitations}
  The synthetic trigger-conditional task, while empirically grounded in real sink behavior \citep{barbero2025llmsattendtoken,Guo2024ActiveDormantAH}, represents a specific computational pattern within a broader class of trigger-conditional problems.
  Our analysis likely extends to related tasks such as key-query retrieval where a query must extract a specific previous token (e.g., marked by a feature bit) while ignoring others—resembling the apostrophe head in \cref{fig:apostrophe}.
  We leave the formal characterization of the full class of tasks necessitating sinks to future work.

  For multi-layer models, our necessity result (\cref{thm:multilayer}) guarantees that at least one layer must exhibit sink behavior at some non-trigger position, but does not characterize which specific layer this must be.
  Our experiments extend this to multi-head architectures and confirm that sinks indeed do not form in all heads or layers (\cref{sec:appendix-additional-exp}), consistent with the existential nature of the theorem; understanding exactly where sinks emerge would likely require a dynamical analysis of how optimization selects among valid solutions, which we leave to future work.

  Finally, it would be interesting to investigate whether other special tokens that are stable and always present in the input (e.g., \texttt{<|think|>} in reasoning models) exhibit similar sink behavior, and to investigate the relatively newly discovered phenomenon of ``secondary attention sinks'' \citep{wong2026existencebehaviorsecondaryattention}. We leave this direction for future work as well.

\section*{Acknowledgments}
I thank Yotam Alexander, Amit Elhelo, Daniela Gottesman, Eden Lumbroso and Yoni Slutsky for illuminating discussions. Special thanks to my advisor Nadav Cohen for his guidance and mentorship.  We used AI assistance for writing and code development. This work was supported by the European Research Council (ERC) grant NN4C 101164614, a Google Research Scholar Award, a Google Research Gift, Meta, the Yandex Initiative in Machine Learning, the Israel Science Foundation (ISF) grant 1780/21, the Tel Aviv University Center for AI and Data Science, the Adelis Research Fund for Artificial Intelligence, Len Blavatnik and the Blavatnik Family Foundation, and Amnon and Anat Shashua.

\bibliography{custom}

\appendix

\section{Training Details}\label{sec:appendix-training}

All models are trained using the Adam optimizer ($\beta_1{=}0.9, \beta_2{=}0.95$) with batch size $128$ over the $\ell_2$ loss until the $\ell_\infty$ loss is less than $ 10^{-2}$ for the entire batch. Single-layer models use learning rate $10^{-3}$; multi-layer models use learning rate $10^{-4}$.
  We use input dimension $n=16$ and sample content coordinates i.i.d. from $\mathcal{U}(-1,1)$.

\section{Practical Impact of Attention Sinks}\label{sec:appendix-sink-impact}

The goal of this section is to detail the empirical motivation for our theoretical study.
Attention sinks have been shown to affect several aspects of model performance and deployment.
We briefly survey the evidence here to motivate the practical importance of understanding their origin.

\paragraph{Accuracy and context utilization.}
When probability mass concentrates on a fixed position, attention can be diverted away from other tokens and downstream accuracy can be affected \citep{Yu2024Unveiling}.
\citet{Guo2024ActiveDormantAH} document ``active--dormant'' heads in which dormant sink behavior effectively wastes representational capacity.

\paragraph{Compression and quantization.}
Attention sinks are correlated with outlier activations that complicate model compression.
\citet{sun2024massive} identify massive activations tied to sink tokens, and \citet{lin2024duquant} show that these outliers are a key challenge for quantization.
\paragraph{Streaming and long-context inference.}
Attention sinks complicate streaming and rolling-window KV-cache strategies: \citet{xiao2023efficient} show that evicting sink tokens from the cache causes catastrophic performance degradation, and that explicitly retaining them is necessary for stable generation on sequences far beyond the training length.

\paragraph{Vision and multimodal models.}
Analogous sink effects appear in vision Transformers and multimodal models.
\citet{Kang2025See} show that visual attention sinks allocate high attention weights to irrelevant visual tokens, wasting representational capacity.
\citet{wang2025mirage} demonstrate that attention sinks in multimodal models can be exploited to induce hallucinations, and \citet{Feng2025EDIT} propose architectural modifications to mitigate sink behavior in vision Transformers.

\paragraph{Interpretability.}
Sinks distort attention-based analyses by concentrating probability mass on tokens that carry no content-relevant information, complicating efforts to use attention patterns for model interpretation \citep{Guo2024ActiveDormantAH}.

\section{Additional Experimental Results}\label{sec:appendix-additional-exp}

To further validate our findings at larger scale, we train 4-layer 4-head models with both softmax and ReLU attention. All models use the same training configuration described in Appendix~\ref{sec:appendix-training}.
  \Cref{fig:appendix-softmax-4h4d,fig:appendix-relu-4h4d} show representative attention patterns.
  The softmax variant exhibits strong sink behavior in at least one head per layer in the no-trigger regime, while the ReLU variant maintains near-zero attention on \texttt{BOS} throughout.
  These results provide additional evidence that the necessity of attention sinks in softmax models persists in deeper, wider architectures.

\section{Proof of \cref{thm:main}}\label{proof:main}
  
  We prove \cref{thm:main} by establishing two separate necessity results: one for pre-trigger positions (\cref{thm:pre-trigger-app}) and one for post-trigger positions (\cref{thm:post-trigger-app}).
  Combining these two results directly yields the statement of \cref{thm:main}, which asserts necessity at all non-trigger positions $i \neq j$.

  \begin{theorem}[Pre-Trigger Necessity]\label{thm:pre-trigger-app}
    For any $\varepsilon, \delta \in \mathbb{R}_{>0}, L \in \mathbb{N}_{\geq 4}$, $n \in \mathbb{N}_{\geq 5}$, and a bounded probability density function $\mathcal{P}$, there exists a constant $\eta \in \mathbb{R}_{>0}$ such that the following holds. Consider any single-layer softmax attention model $f$ with loss $\mathcal{L}(f) \leq \eta$ on sequences with length $L$ and dimension $n$ where non-trigger coordinates are drawn from $\mathcal{P}$. Then with probability at least $1 - \delta$ over the choice of $\mathbf{x}$ with trigger position $j$, for all pre-trigger positions $1 < i < j$, we have $\alpha_{i,1} \geq 1 - \varepsilon$.
  \end{theorem}
  \begin{proof}

    \medskip\noindent\textbf{Step 1: We can assume that $\mathbf{W}_K=\mathbf{I}$ and $\mathbf{W}_O=\mathbf{I}$.}
  Let
  \[
  \mathbf{B}:=\mathbf{W}_Q\mathbf{W}_K^\top,\qquad
  \mathbf{V}:=\mathbf{W}_O\mathbf{W}_V.
  \]
  For any input, the scores and outputs are
  \[
  s_{i,k}
  = \mathbf{x}^{(i)}\mathbf{B}(\mathbf{x}^{(k)})^\top,
  \qquad
  \hat{\mathbf{y}}^{(i)}
  = \sum_{k\le i}\alpha_{i,k}\,\mathbf{V}\,\mathbf{x}^{(k)},
  \]
  with
  \[
  \alpha_{i,k}
  = \frac{\exp(s_{i,k})}{\sum_{\ell\le i}\exp(s_{i,\ell})}.
  \]
  Thus the attention depends on $(\mathbf{W}_Q,\mathbf{W}_K)$ only through
  $\mathbf{B}$, and the output depends on $(\mathbf{W}_O,\mathbf{W}_V)$ only  
  through $\mathbf{V}$.
  Reparameterizing by setting
  \begin{align*}
  \mathbf{W}_K&:=\mathbf{I},\quad\mathbf{W}_Q:=\mathbf{B},\\
  \mathbf{W}_O&:=\mathbf{I},\quad\mathbf{W}_V:=\mathbf{V}
  \end{align*}
  leaves $\alpha_{i,k}$ and $\hat{\mathbf{y}}^{(i)}$ unchanged, hence the
  loss is unchanged.
  Therefore, we will assume without loss of generality that $\mathbf{W}_K=\mathbf{I}$ and
  $\mathbf{W}_O=\mathbf{I}$, write $\mathbf{Q}$ for the query map, and
  $\mathbf{V}$ for the (combined) value map.

\medskip\noindent\textbf{Step 2: Setup and pigeonhole principle.}
Fix $\varepsilon_0,\delta_0 \in \mathbb{R}_{>0}$ and suppose by contradiction that there exists a sequence of one-layer softmax models $\{f_t\}_{t=1}^\infty$ with $\eta_t:=\mathcal{L}(f_t)\to 0$ such that, for each $t$, with probability at least $\delta_0$ over $(\mathbf{x},j)\sim\mathcal{P}$ there is a pre-trigger position $i<j$ violating the sink condition:
\begin{equation}\label{eq:nosink}
\alpha_{i,1} \le 1-\varepsilon_0.
\end{equation}
Since $\sum_{k\le i}\alpha_{i,k}=1$, \eqref{eq:nosink} implies that the total mass on non-\texttt{BOS} keys is at least $\varepsilon_0$.
  There are only finitely many position triples $(i,h,j)$ with $2\le h\le i<j\le L$.
  By a pigeonhole principle, there exist infinitely many times $t_{a_1},t_{a_2},\ldots$ and fixed indices $2\le i^\star < j^\star\le L$ and $2\le h^\star\le i^\star$, and a constant $\gamma \in \mathbb{R}_{>0}$ (e.g., $\gamma=\varepsilon_0/L^2$), such that
\begin{equation}\label{eq:fixed-triple}
\mathbb{P}\!\Big(\alpha_{i^\star,1}\le 1-\varepsilon_0 \text{ and } \alpha_{i^\star,h^\star}\ge \gamma\Big) \ge \delta
\end{equation}
for some $\delta \in \mathbb{R}_{>0}$ independent of $t$.
  By relabeling this subsequence, we assume without loss of generality that \eqref{eq:fixed-triple} holds for all $t$.

\medskip\noindent\textbf{Step 3: Constructing tokens via Lemma~\ref{lem:axis-separated}.}

Since the event in \eqref{eq:fixed-triple} has positive probability at least $\delta$, by Lemma~\ref{lem:axis-separated} (applied to content coordinates) there exists $\varepsilon' \in \mathbb{R}_{>0}$ (independent of $t$) such that for every content coordinate $m\in\{4,\dots,n\}$ there exist tokens $x^{(m)},y^{(m)}$ with the following properties: (i) $x^{(m)}_k = y^{(m)}_k$ for all $k\neq m$, and $\bigl|x^{(m)}_m - y^{(m)}_m\bigr| \ge \varepsilon'$; and (ii) there exist sequences with either $x^{(m)}$ or $y^{(m)}$ at position $h^\star$ and with trigger position $j$ satisfying $i^\star<j$, such that
\begin{align}\label{eq:alpha-lb-interval}
\alpha_{i^\star,h^\star} \ge \gamma.
\end{align}

\medskip\noindent\textbf{Step 4: Positive weight implies small values.}
By Lemma~\ref{lem:pairwise-O-eta} (applied with the pair $(h^\star,i^\star)$ in the case where $h^\star \neq i^\star$) and Lemma~\ref{lem:self-O-eta} (applied with $h^\star$ whenever $h^\star = i^\star$), for every choice of token at position $h^\star$ we have
\[
\big\|\alpha_{i^\star,h^\star}\mathbf{V}\mathbf{x}^{(h^\star)}\big\|_2 \le 4\eta_t.
\]
Combining with \eqref{eq:alpha-lb-interval} yields that for any content coordinate $m$ and any $\mathbf{z}\in\{x^{(m)},y^{(m)}\}$,
\begin{equation}\label{eq:V-small}
\big\|\mathbf{V}\mathbf{z}\big\|_2 \le \tfrac{4}{\gamma}\eta_t.
\end{equation}
That is, the lower bound on $\alpha_{i^\star,h^\star}$ directly forces the value projections to be small for all tokens constructed in Step~2.

\medskip\noindent\textbf{Step 5: Transplanting to $j=3$ and deriving a contradiction.}
Fix $t$ and abbreviate $\eta:=\eta_t$.
  Pick a content coordinate $m\in\{4,\ldots,n\}$ and let $\mathbf{x}_t:=x^{(m)}$ and $\mathbf{y}_t:=y^{(m)}$ be the two tokens from Step~2 satisfying $|\mathbf{x}_{t}^{(m)}-\mathbf{y}_{t}^{(m)}|\ge\varepsilon'$.
  Instantiate two sequences by setting the trigger at $j=3$, taking $\mathbf{x}^{(2)}\in\{\mathbf{x}_t,\mathbf{y}_t\}$, and fixing the trigger token $\mathbf{x}^{(3)}$ to any arbitrary value $\mathbf{t}$ such that the sequence is in the support of $\mathcal{D}$.
  At position $i=3$ the target is
\begin{align}
\mathbf{y}^{(3)} &= \frac{1}{2}(\mathbf{x}^{(2)} + \mathbf{t}).
\end{align}

For any $\mathbf{z}\in\{\mathbf{x}_t,\mathbf{y}_t\}$, let $\beta_t(\mathbf{z})$ be the attention weight $\alpha_{3,3}$ computed on the sequence where $\mathbf{x}^{(2)}=\mathbf{z}$ and $\mathbf{x}^{(3)}=\mathbf{t}$. Define the fixed value vector
\begin{align}
\mathbf{v}_t \;&:=\; \mathbf{V}_t\,\mathbf{t}.
\end{align}
By Lemma~\ref{lem:bos-small} and \eqref{eq:V-small}, at position $3$ we can decompose
\begin{align}
\hat{\mathbf{y}}^{(3)}(\mathbf{z})
&= \underbrace{\alpha_{3,1}\,\mathbf{V}e_1 + \alpha_{3,2}\,\mathbf{V}\mathbf{z}}_{=:~\mathbf{r}_t(\mathbf{z})}
+ \beta_t(\mathbf{z})\,\mathbf{v}_t, \label{eq:haty3-decomp}\\
\|\mathbf{r}_t(\mathbf{z})\|_2 &\le C_0\,\eta,
\end{align}
with $C_0:=1+\tfrac{4}{\gamma}$ independent of $t$.
  Consider coordinate~$3$ (the non-trigger non-BOS indicator).
  Since $(\mathbf{y}^{(3)})_3 = \frac{1}{2}((\mathbf{x}^{(2)})_3 + (\mathbf{t})_3) = \frac{1}{2}(1 + 0) = 0.5$ and $0<\beta_t(\mathbf{z})\le 1$, from \eqref{eq:haty3-decomp} and the uniform loss bound we obtain
\begin{align}
&\big|\beta_t(\mathbf{z})\,(\mathbf{v}_t)_3 - 0.5\big| \notag\\
&\quad\le \big|\hat{\mathbf{y}}^{(3)}_3(\mathbf{z}) - 0.5\big|
+ \big|(\mathbf{r}_t(\mathbf{z}))_3\big| \notag\\
&\quad\le \eta + C_0\eta \notag\\
&\quad= C_1\eta,
\end{align}
where $C_1:=1+C_0$.
  Hence, for all sufficiently large $t$,
\begin{align}
(\mathbf{v}_t)_3
\;&\ge\; \frac{0.5-C_1\eta}{\beta_t(\mathbf{z})}
\;\ge\; 0.5-C_1\eta
\;>\; 0, \label{eq:vn-positive}
\end{align}
so $\mathbf{v}_t\neq \mathbf{0}$.

Let $P_t$ denote the orthogonal projection onto $\mathbf{v}_t^\perp$.
  Since $P_t$ is an orthogonal projection onto an $(n-1)$-dimensional subspace, there must be at least one coordinate $m_0 \in \{4,5\}$ such that $\|P_t e_{m_0}\|_2 \geq 1/\sqrt{2}$; fix $m$ to be that coordinate.
Now, applying $P_t$ to \eqref{eq:haty3-decomp} kills the $\mathbf{v}_t$ component:
\begin{align}
P_t\hat{\mathbf{y}}^{(3)}(\mathbf{z}) &= P_t\mathbf{r}_t(\mathbf{z}), \\
\|P_t\hat{\mathbf{y}}^{(3)}(\mathbf{z})\|_2
&\le \|\mathbf{r}_t(\mathbf{z})\|_2
\le C_0\eta.
\end{align}
Therefore, for the two choices $\mathbf{z}=\mathbf{x}_t,\mathbf{y}_t$,
\begin{align}
&\big\|P_t\hat{\mathbf{y}}^{(3)}(\mathbf{x}_t) - P_t\hat{\mathbf{y}}^{(3)}(\mathbf{y}_t)\big\|_2 \notag\\
&\quad\le\; \|P_t\mathbf{r}_t(\mathbf{x}_t)\|_2 + \|P_t\mathbf{r}_t(\mathbf{y}_t)\|_2 \notag\\
&\quad\le\; 2C_0\eta. \label{eq:proj-pred-diff}
\end{align}
On the other hand, we have $\mathbf{y}^{(3)}(\mathbf{z}) = \frac{1}{2}(\mathbf{z}+\mathbf{t})$, so $P_t\mathbf{y}^{(3)}(\mathbf{z}) = \frac{1}{2}P_t\mathbf{z} + \frac{1}{2}P_t\mathbf{t}$. Since the $\mathbf{t}$ term is constant in $\mathbf{z}$, it cancels in the difference:
\begin{align}
&\big\|P_t\mathbf{y}^{(3)}(\mathbf{x}_t) - P_t\mathbf{y}^{(3)}(\mathbf{y}_t)\big\|_2 \notag\\
&\quad= \frac{1}{2}\|P_t(\mathbf{x}_t-\mathbf{y}_t)\|_2 \notag\\
&\quad= \frac{1}{2}\|P_t((\mathbf{x}_{t,m}-\mathbf{y}_{t,m})e_m)\|_2 \notag\\
&\quad = \frac{1}{2}|\mathbf{x}_{t,m}-\mathbf{y}_{t,m}|\,\|P_t e_m\|_2 \notag\\
&\quad\ge \frac{1}{2}\varepsilon'\,\|P_t e_m\|_2 \notag\\
&\quad\ge \frac{1}{2\sqrt{2}}\varepsilon'.\label{eq:proj-target-diff}
\end{align}
where the third equality uses the fact that $\mathbf{x}_t$ and $\mathbf{y}_t$ differ only on coordinate $m$.

Finally, by the triangle inequality and the uniform loss bound,
\begin{align}
&\big\|P_t\mathbf{y}^{(3)}(\mathbf{x}_t) - P_t\mathbf{y}^{(3)}(\mathbf{y}_t)\big\|_2 \notag\\
&\quad\le \big\|P_t\hat{\mathbf{y}}^{(3)}(\mathbf{x}_t) - P_t\hat{\mathbf{y}}^{(3)}(\mathbf{y}_t)\big\|_2 + 2\eta \notag\\
&\quad\le (2C_0+2)\eta,
\end{align}
which contradicts \eqref{eq:proj-target-diff} for all sufficiently small $\eta$, because $\varepsilon'\|P_t e_m\|_2>0$ is independent of $t$.

    \end{proof}

  \begin{theorem}[Post-Trigger Necessity]\label{thm:post-trigger-app}
    For any $\varepsilon, \delta \in \mathbb{R}_{>0}, L \in \mathbb{N}_{\geq 4}$, $n \in \mathbb{N}_{\geq 5}$, and a bounded probability density function $\mathcal{P}$, there exists a constant $\eta \in \mathbb{R}_{>0}$ such that the following holds. Consider any single-layer softmax attention model $f$ with loss $\mathcal{L}(f) \leq \eta$ on sequences with length $L$ and dimension $n$ where non-trigger coordinates are drawn from $\mathcal{P}$. Then with probability at least $1 - \delta$ over the choice of $\mathbf{x}$ with trigger position $j$, for all post-trigger positions $j < i \leq L$, we have $\alpha_{i,1} \geq 1 - \varepsilon$.
  \end{theorem}

  \begin{proof}
    \medskip\noindent\textbf{Step 1: The trigger receives arbitrarily small attention post-trigger.}
Fix any trigger token $\mathbf{t}$ and any non-trigger token $\mathbf{z}$, and consider the length-$3$ prefix
$(\texttt{BOS},\mathbf{t},\mathbf{z})$ (so the trigger position is $j=2$ and position $3$ is post-trigger). Let
$\widetilde{\alpha}_{3,1},\widetilde{\alpha}_{3,2},\widetilde{\alpha}_{3,3}$ be the attention weights at position $3$.

We first bound the self term $\widetilde{\alpha}_{3,3}\mathbf{V}\mathbf{z}$ using Lemma~\ref{lem:self-O-eta}.
Embed the pair $(\texttt{BOS},\mathbf{z})$ as the first two tokens of any valid sequence from $\mathcal{D}$ whose trigger position
satisfies $j\ge 3$ (so position $2$ is pre-trigger and non-trigger). Applying Lemma~\ref{lem:self-O-eta} at $i=2$ gives
$\|\alpha_{2,2}\mathbf{V}\mathbf{z}\|_2\le 2\eta$ for that sequence, and by Lemma~\ref{lem:softmax-monotone} (adding the extra key
$\mathbf{t}$ can only decrease the probability assigned to $\mathbf{z}$) we have $\widetilde{\alpha}_{3,3}\le \alpha_{2,2}$, hence
\(
\|\widetilde{\alpha}_{3,3}\mathbf{V}\mathbf{z}\|_2 \le 2\eta.
\)
Also, Lemma~\ref{lem:bos-small} gives $\|\widetilde{\alpha}_{3,1}\mathbf{V}e_1\|_2\le \eta$.
Since the target at position $3$ is $\mathbf{0}$, we have $\|\hat{\mathbf{y}}^{(3)}\|_2\le \eta$, and therefore
\begin{align*}
\|\widetilde{\alpha}_{3,2}\mathbf{V}\mathbf{t}\|_2
&\le \|\hat{\mathbf{y}}^{(3)}\|_2+\|\widetilde{\alpha}_{3,1}\mathbf{V}e_1\|_2\\
&\quad+\|\widetilde{\alpha}_{3,3}\mathbf{V}\mathbf{z}\|_2 \le 4\eta.
\end{align*}
Finally, Lemma~\ref{lem:trigger-large} (applied to any valid sequence with trigger at position $2$) gives
$\|\mathbf{V}\mathbf{t}\|_2\ge 1-2\eta$, so
\begin{equation}\label{eq:alpha32-small-post-final}
\widetilde{\alpha}_{3,2} \le \frac{4\eta}{1-2\eta}.
\end{equation}

Now fix any valid sequence $\mathbf{x}$ with trigger position $j$ and any $i>j$. By Lemma~\ref{lem:reductions}(2),
$\alpha_{i,j}\le \widehat{\alpha}_{3,2}$ where $\widehat{\alpha}_{3,2}$ is the attention weight on the second token in the prefix
$(\texttt{BOS},\mathbf{x}^{(j)},\mathbf{x}^{(i)})$, and applying \eqref{eq:alpha32-small-post-final} to that prefix yields
\begin{equation}\label{eq:alphaij-small-post-final}
\alpha_{i,j} \le \frac{4\eta}{1-2\eta}\qquad\text{for all }i>j.
\end{equation}
    
\medskip\noindent\textbf{Step 2 (contradiction via shifting the trigger).}
Fix $\varepsilon_0,\delta_0 \in \mathbb{R}_{>0}$ and suppose, for contradiction, that the theorem is false. Then there exists a sequence of
one-layer softmax models $\{f_t\}_{t\ge 1}$ with $\eta_t:=\mathcal{L}(f_t)\to 0$ such that for every $t$,
\begin{equation}\label{eq:post-fail}
\mathbb{P}_{(\mathbf{x},j)\sim\mathcal{D}}\Big(\exists\, i>j:\ \alpha^{(t)}_{i,1}(\mathbf{x})\le 1-\varepsilon_0\Big)\ \ge\ \delta_0.
\end{equation}

By Step~1 (i.e., \eqref{eq:alphaij-small-post-final}), for all $\mathbf{x}$ in $\mathrm{support}(\mathcal{D})$ and all $i>j$,
\[
\alpha^{(t)}_{i,j}(\mathbf{x}) \le \frac{4\eta_t}{1-2\eta_t}.
\]
Fix $t$ large enough so that $\frac{4\eta_t}{1-2\eta_t}\le \varepsilon_0/2$. 

Let $E_t$ be the event in \eqref{eq:post-fail}. For each $(\mathbf{x},j)\in E_t$ there exists $i(\mathbf{x})>j$ with $\alpha^{(t)}_{i(\mathbf{x}),1}(\mathbf{x})\le 1-\varepsilon_0$.
For that $i(\mathbf{x})$ we also have $\alpha^{(t)}_{i(\mathbf{x}),j}(\mathbf{x})\le \varepsilon_0/2$, hence
\begin{align}\label{eq:mass-nontrigger}
&\sum_{k\le i(\mathbf{x}),\ k\notin\{1,j\}}\alpha^{(t)}_{i(\mathbf{x}),k}(\mathbf{x}) \notag\\
&\quad= 1-\alpha^{(t)}_{i(\mathbf{x}),1}(\mathbf{x})-\alpha^{(t)}_{i(\mathbf{x}),j}(\mathbf{x}) \notag\\
&\quad\ge \varepsilon_0/2.
\end{align}

Now define the shift map $\mathrm{Shift}$ that moves the trigger token to the end:
$\mathbf{x}' = \mathrm{Shift}(\mathbf{x}, j)$, where $\mathbf{x}'^{(k)} = \mathbf{x}^{(k)}$ for $1 \le k < j$ (positions before the trigger are unchanged), $\mathbf{x}'^{(k)} = \mathbf{x}^{(k+1)}$ for $j \le k \le L-1$ (positions after the trigger shift left by one), and $\mathbf{x}'^{(L)} = \mathbf{x}^{(j)}$ (trigger moves to the end).
Then $\mathbf{x}'\in\mathrm{support}(\mathcal{D})$ with trigger position $L$, and moreover, by the definition of the task (\cref{sec:task-def}), we have that the probability density of $\mathbf{x}'$ is the same as that of $\mathbf{x}$:
\begin{equation}\label{eq:shift-probability-density}
\mathcal{P}(\mathbf{x}) = \mathcal{P}(\mathbf{x}').
\end{equation}

Fix $(\mathbf{x},j)\in E_t$.
Applying Lemma~\ref{lem:softmax-monotone} we get that removing the key $j$
can only \emph{increase} the attention weight of each remaining key (at position $i(\mathbf{x})$ in $\mathbf{x}$, the ``candidate'' key set is $\{1,\dots,i(\mathbf{x})\}$; at position $i(\mathbf{x})-1$ in $\mathbf{x}'$, the ``candidate'' key set
is $\{1,\dots,i(\mathbf{x})\}\setminus\{j\}$, the same set with the trigger key removed.)
Therefore,
\begin{align*}
&\sum_{r\le i(\mathbf{x})-1,\ r\ne 1}\alpha'^{(t)}_{i(\mathbf{x})-1,r}(\mathbf{x}') \\
&\quad\ge \sum_{k\le i(\mathbf{x}),\ k\notin\{1,j\}}\alpha^{(t)}_{i(\mathbf{x}),k}(\mathbf{x}) \\
&\quad\ge\ \varepsilon_0/2,
\end{align*}
where the last inequality is \eqref{eq:mass-nontrigger}. Equivalently,
\begin{equation}\label{eq:shift-violates-pre}
\alpha'^{(t)}_{i(\mathbf{x})-1,1}(\mathbf{x}')\ \le\ 1-\varepsilon_0/2.
\end{equation}

Since this holds for \emph{every} $(\mathbf{x},j)\in E_t$, by the Pigeonhole Principle, there exist fixed indices $j^* \in \{1,\dots,L\}$ and $r^* < L$ and a constant $c_1 \in \mathbb{R}_{>0}$ such that for infinitely many $t$,
\begin{equation}\label{eq:shifted-failure-prob}
\begin{split}
&\mathbb{P}_{(\mathbf{x},j)\sim\mathcal{D}}\Big(\alpha'^{(t)}_{r^*,1}(\mathbf{x}')\le 1-\varepsilon_0/2 \ \text{and}\ j=j^*\Big)\\
&\qquad\ge\ c_1 \delta_0,
\end{split}
\end{equation}
where $\mathbf{x}'=\mathrm{Shift}(\mathbf{x},j)$.

Finally, consider the bijection $\mathbf{x} \mapsto \mathrm{Shift}(\mathbf{x}, j^*)$ from the set of sequences with trigger at $j^*$ to the set of sequences with trigger at $L$. By \cref{eq:shift-probability-density}, this map preserves probability density. Thus, the event in \eqref{eq:shifted-failure-prob} has the exact same probability as the corresponding event for sequences with trigger at $L$:
\begin{align*}
&\mathbb{P}_{(\mathbf{z},j)\sim\mathcal{D}}\Big(\alpha^{(t)}_{r^*,1}(\mathbf{z})\le 1-\varepsilon_0/2 \ \text{and}\ j=L\Big)\\
&\quad\ge\ c_1 \delta_0.
\end{align*}
Conditioning on $j=L$, this implies that for infinitely many $t$,
\[
\mathbb{P}\Big(\alpha^{(t)}_{r^*,1}(\mathbf{z})\le 1-\varepsilon_0/2 \ \Big|\ j=L\Big) \ge \frac{c_1 \delta_0}{\mathbb{P}(j=L)}.
\]
Since $r^* < L$, this contradicts Theorem~\ref{thm:pre-trigger-app}, as needed.

\end{proof}

\section{Proof of \cref{thm:multilayer}}\label{proof:multilayer}

\paragraph{Step 1: Setup and contradiction assumption.}
  Fix $\varepsilon_0,\delta_0 \in \mathbb{R}_{>0}$.
  Suppose for contradiction that there exists a sequence of $D$-layer softmax models $\{f_t\}_{t=1}^\infty$ with
  \[
  \eta_t \;:=\; \mathcal{L}(f_t)\;\longrightarrow\;0
  \]
  such that, for every $t$,
  \begin{equation}\label{eq:multilayer-no-sink-event}
  \begin{split}
  &\mathbb{P}\!\Big(\forall d\in\{1,\ldots,D\},\ \forall\,1<i<j:\ \\
  &\qquad\alpha^{(d)}_{i,1}\le 1-\varepsilon_0 \Big | j \geq 3\Big)\;\ge\;\delta_0.
  \end{split}
  \end{equation}
  Let $E_t$ denote the event inside the probability in \eqref{eq:multilayer-no-sink-event} intersected with the event $j \geq 3$.
  For each $t$, let $\mathbf{V}_t$ be the combined value map from Lemma~\ref{lem:multilayer-unroll}, and write $\beta^{(t)}_{i,k}(\cdot)$ for the corresponding coefficients.

\paragraph{Step 2: No sink implies small value projections.}
  On the event $E_t$, position $2$ is pre-trigger (since $j\ge 3$) and for every layer $d$,
  \[
  \alpha^{(d)}_{2,2} \;=\; 1-\alpha^{(d)}_{2,1} \;\ge\; \varepsilon_0.
  \]
  Therefore, by Lemma~\ref{lem:beta22-product} conditioned on $E_t$ we have that
  \begin{equation}\label{eq:beta22-lowerbound}
  \beta^{(t)}_{2,2}(\mathbf{x}) \;\ge\; \varepsilon_0^{D}
  \end{equation}
  Moreover, Lemma~\ref{lem:beta22-O-eta} applied to $f_t$ yields
  \[
  \big\|\beta^{(t)}_{2,2}(\mathbf{x})\,\mathbf{V}_t\mathbf{x}^{(2)}\big\|_2 \;\le\; 2\eta_t.
  \]
  Combining with \eqref{eq:beta22-lowerbound} gives
  \begin{equation}\label{eq:Vx2-small-multi}
  \|\mathbf{V}_t\mathbf{x}^{(2)}\|_2 \;\le\; \frac{2}{\varepsilon_0^{D}}\,\eta_t
  \qquad\text{on }E_t.
  \end{equation}
  Define the measurable set
  \[
  S_t
  \;:=\;
  \Big\{\mathbf{z}\in\mathbb{R}^n:\ \|\mathbf{V}_t\mathbf{z}\|_2 \le \tfrac{2}{\varepsilon_0^{D}}\eta_t\Big\}.
  \]
  Since $E_t\subseteq\{\mathbf{x}^{(2)}\in S_t\}$ by \eqref{eq:Vx2-small-multi}, \eqref{eq:multilayer-no-sink-event} implies
  \begin{equation}\label{eq:St-positive-mass}
  \mathbb{P}\big(\mathbf{x}^{(2)}\in S_t \big | j \geq 3 \big)\;\ge\;\delta_0.
  \end{equation}
  
  By Lemma~\ref{lem:axis-separated} (applied to content coordinates) and \eqref{eq:St-positive-mass}, there exists $\varepsilon' \in \mathbb{R}_{>0}$ (independent of $t$) such that for every content coordinate $m\in\{4,\ldots,n\}$ there exist tokens
  $\mathbf{x}_t^{(m)},\mathbf{y}_t^{(m)}\in S_t$ satisfying
  \begin{equation}\label{eq:xtyt-separated}
  \begin{split}
  &\mathbf{x}^{(m)}_{t,k}=\mathbf{y}^{(m)}_{t,k}\ \text{for all }k\neq m, \\
  &\quad
  \big|\mathbf{x}^{(m)}_{t,m}-\mathbf{y}^{(m)}_{t,m}\big|\;\ge\;\varepsilon'.
  \end{split}
  \end{equation}

\paragraph{Step 3: Transplanting to $j=3$ and deriving a contradiction.}
  Fix $t$ and abbreviate $\eta:=\eta_t$.
  Pick a content coordinate $m\in\{4,\ldots,n\}$ and let $\mathbf{x}_t:=\mathbf{x}_t^{(m)}$ and $\mathbf{y}_t:=\mathbf{y}_t^{(m)}$ be the two tokens from Step~2 satisfying $|\mathbf{x}_{t,m}-\mathbf{y}_{t,m}|\ge\varepsilon'$.
  Instantiate two sequences by setting the trigger at $j=3$, taking $\mathbf{x}^{(2)}\in\{\mathbf{x}_t,\mathbf{y}_t\}$, and fixing the trigger token $\mathbf{x}^{(3)}$ to an arbitrary value $\mathbf{t}$ such that the sequence is in the support of $\mathcal{D}$.
  At position $i=3$ the target is
  \begin{equation}\label{eq:target-j3}
  \mathbf{y}^{(3)} \;=\; \frac{1}{2}(\mathbf{x}^{(2)} + \mathbf{t}).
  \end{equation}
  
  For any $\mathbf{z}\in\{\mathbf{x}_t,\mathbf{y}_t\}$, let $\beta_t(\mathbf{z})$ be the coefficient $\beta^{(t)}_{3,3}(\mathbf{z})$ computed on the sequence where $\mathbf{x}^{(2)}=\mathbf{z}$ and $\mathbf{x}^{(3)}=\mathbf{t}$. Define the fixed value vector
  \begin{align}
  \mathbf{v}_t \;&:=\; \mathbf{V}_t\,\mathbf{t}.
  \end{align}
  By Lemma~\ref{lem:multilayer-unroll}, for each choice $\mathbf{x}^{(2)}=\mathbf{z}$ we can decompose
  \begin{align}
  \hat{\mathbf{y}}^{(3)}(\mathbf{z})
  &=
  \underbrace{\beta^{(t)}_{3,1}(\mathbf{z})\,\mathbf{V}_t e_1
  +
  \beta^{(t)}_{3,2}(\mathbf{z})\,\mathbf{V}_t \mathbf{z}}_{=:~\mathbf{r}_t(\mathbf{z})}
  +
  \beta_t(\mathbf{z})\,\mathbf{v}_t.
  \label{eq:haty3-decomp-multi}
  \end{align}
  Since $\beta^{(t)}_{3,1}(\mathbf{z}),\beta^{(t)}_{3,2}(\mathbf{z})\le 1$, Lemma~\ref{lem:bos-small-multilayer} gives $\|\mathbf{V}_t e_1\|_2\le \eta$, and $\mathbf{z}\in S_t$ implies $\|\mathbf{V}_t\mathbf{z}\|_2\le \tfrac{2}{\varepsilon_0^D}\eta$.
  Therefore
  \begin{equation}\label{eq:rt-small-multi}
  \|\mathbf{r}_t(\mathbf{z})\|_2 \;\le\; C_0\,\eta,
  \qquad
  C_0 \;:=\; 1+\tfrac{2}{\varepsilon_0^D}.
  \end{equation}
  
  Consider coordinate~$3$ (the non-trigger non-BOS indicator).
  For the $j=3$ construction, we have $(\mathbf{y}^{(3)})_3=0.5$.
  Using \eqref{eq:haty3-decomp-multi} and the uniform loss bound,
  \begin{align*}
  &\big|\beta_t(\mathbf{z})\,(\mathbf{v}_t)_3 - 0.5\big| \\
  &\quad\le
  \big|\hat{\mathbf{y}}^{(3)}_3(\mathbf{z})-0.5\big|
  +
  \big|(\mathbf{r}_t(\mathbf{z}))_3\big| \\
  &\quad\le \eta+C_0\eta \\
  &\quad= C_1\eta,
  \end{align*}
  where $C_1:=1+C_0$.
  Hence $(\mathbf{v}_t)_3\ge 0.5-C_1\eta>0$ for all sufficiently large $t$, so $\mathbf{v}_t\neq \mathbf{0}$.
  
  Let $P_t$ denote the orthogonal projection onto $\mathbf{v}_t^\perp$.
  Since $\dim(\mathbf{v}_t^\perp)=n-1$, there exists at least one coordinate $m_0 \in\{4,5\}$ such that
  \begin{equation}\label{eq:proj-em-lb-multi}
  \|P_t e_{m_0}\|_2 \;\ge\; 1/\sqrt{2}.
  \end{equation}
  Fix such an $m$, and take $\mathbf{x}_t:=\mathbf{x}_t^{(m)}$ and $\mathbf{y}_t:=\mathbf{y}_t^{(m)}$ from \eqref{eq:xtyt-separated}.
  
  Applying $P_t$ to \eqref{eq:haty3-decomp-multi} kills the $\mathbf{v}_t$ component, giving
  $P_t\hat{\mathbf{y}}^{(3)}(\mathbf{z})=P_t\mathbf{r}_t(\mathbf{z})$.
  Therefore,
  \begin{align}
  &\big\|P_t\hat{\mathbf{y}}^{(3)}(\mathbf{x}_t)-P_t\hat{\mathbf{y}}^{(3)}(\mathbf{y}_t)\big\|_2 \notag\\
  &\qquad\le \|P_t\mathbf{r}_t(\mathbf{x}_t)\|_2 + \|P_t\mathbf{r}_t(\mathbf{y}_t)\|_2
  \le 2C_0\eta,
  \label{eq:proj-pred-diff-multi}
  \end{align}
  using \eqref{eq:rt-small-multi}.
  On the other hand, by \eqref{eq:target-j3} we have $P_t\mathbf{y}^{(3)}(\mathbf{z}) = \frac{1}{2} P_t(\mathbf{z} + \mathbf{t})$, so
  \begin{align}
  &\big\|P_t\mathbf{y}^{(3)}(\mathbf{x}_t)-P_t\mathbf{y}^{(3)}(\mathbf{y}_t)\big\|_2 \notag\\
  &\qquad=
  \frac{1}{2}\|P_t(\mathbf{x}_t-\mathbf{y}_t)\|_2 \notag\\
  &\qquad=
  \frac{1}{2}|\mathbf{x}_{t,m}-\mathbf{y}_{t,m}|\cdot \|P_t e_m\|_2 \notag\\
  &\qquad\ge
  \frac{1}{2\sqrt{2}}\varepsilon',
  \label{eq:proj-target-diff-multi}
  \end{align}
  using \eqref{eq:xtyt-separated} and 
  \eqref{eq:proj-em-lb-multi}.
  
  Finally, by the triangle inequality and the uniform loss bound,
  \begin{align*}
  &\big\|P_t\mathbf{y}^{(3)}(\mathbf{x}_t)-P_t\mathbf{y}^{(3)}(\mathbf{y}_t)\big\|_2 \\
  &\qquad\le
  \big\|P_t\hat{\mathbf{y}}^{(3)}(\mathbf{x}_t)-P_t\hat{\mathbf{y}}^{(3)}(\mathbf{y}_t)\big\|_2
  +2\eta \\
  &\qquad\le
  (2C_0+2)\eta,
  \end{align*}
  which contradicts \eqref{eq:proj-target-diff-multi} for all sufficiently small $\eta$.
  This contradiction completes the proof.

  \section{Proof of \cref{thm:relu}}\label{proof:relu}

    We give an explicit zero-loss construction with $\alpha_{i,1}=0$ for all $i$.
    
    \paragraph{Parameters.}
    Set $\mathbf{W}_K=\mathbf{I}$, $\mathbf{W}_V=\mathbf{I}$, and $\mathbf{W}_O=\mathbf{I}$.
    Let $e_r$ denote the $r$-th standard basis vector.
    Recall from \cref{sec:task-def}: coordinate 1 is the BOS indicator; coordinate 2 is the trigger indicator;
    coordinate 3 is the non-trigger non-BOS indicator, with $\mathbf{x}^{(1)}_3=\mathbf{x}^{(j)}_3=0$ and $\mathbf{x}^{(i)}_3=1$ for $i\neq 1,j$.
    Define
    \[ \mathbf{W}_Q \;=\; e_2 (e_2 + e_3)^\top. \]

    \paragraph{Computing the attention weights.}
    Using the ReLU attention formula from \cref{sec:attention-weights}, the unnormalized score from position $i$ to position $k$ is
    \[
    \mathbf{x}^{(i)} \mathbf{W}_Q \mathbf{W}_K^\top (\mathbf{x}^{(k)})^\top
    \;=\; \mathbf{x}^{(i)}_2 \cdot (\mathbf{x}^{(k)}_2 + \mathbf{x}^{(k)}_3).
    \]
    
    Fix a trigger position $j\in\{2,\dots,L\}$. For any non-trigger position $i \neq j$, we have $\mathbf{x}^{(i)}_2=0$, so all scores are zero and hence $\alpha_{i,k}=\mathrm{ReLU}(0)/n_i = 0$ for all $k\le i$. In particular, $\alpha_{i,1}=0$.

    For the trigger position $i=j$, we have $\mathbf{x}^{(j)}_2=1$. The score to position $k$ equals $\mathbf{x}^{(k)}_2 + \mathbf{x}^{(k)}_3$. This is $1$ for non-trigger non-BOS tokens (if such exist) $k\in\{2,\dots,j-1\}$ (where $\mathbf{x}^{(k)}_3=1$) and for the trigger token $k=j$ (where $\mathbf{x}^{(k)}_2=1$). It is $0$ for $k=1$ (\texttt{BOS}). After applying ReLU and dividing by $n_j = j-1$, we obtain
    \begin{align*}
    \alpha_{j,k} &= \tfrac{1}{j-1} \quad \text{for } 2\le k\le j, \\
    \alpha_{j,k} &= 0 \quad\;\;\;\, \text{otherwise}.
    \end{align*}

    \paragraph{Verifying the output.}
    At non-trigger positions, all attention weights are zero, so $f(\mathbf{x})^{(i)} = \mathbf{0} = \mathbf{y}^{(i)}$.
    At the trigger position $i=j$, using $\mathbf{W}_O=\mathbf{W}_V=\mathbf{I}$:
    \begin{align*}
    f(\mathbf{x})^{(j)}
    &= \mathbf{W}_O \sum_{k=1}^{j} \alpha_{j,k}\, \mathbf{W}_V \mathbf{x}^{(k)} \\
    &= \frac{1}{j-1}\sum_{k=2}^{j} \mathbf{x}^{(k)}
    = \overline{\mathbf{x}}
    = \mathbf{y}^{(j)}.
    \end{align*}
    Thus $\mathcal{L}(f)=0$ and $\alpha_{i,1}=0$ for all $i$, completing the proof.

\section{Lemmas}

\begin{lemma}\label{lem:bos-small}
Let $f$ be a single-layer softmax self-attention model as in \S\ref{sec:attention-weights} and write $\mathbf{V}:=\mathbf{W}_O\mathbf{W}_V$. If the loss $\mathcal{L}(f)$ (see \cref{sec:task-def}) satisfies $\mathcal{L}(f)\le \eta$, then
\[
\|\mathbf{V}e_1\|_2 \le \eta .
\]
\end{lemma}

\begin{proof}
By causality, at position $i=1$ we have $\alpha_{1,1}=1$, hence $\hat{\mathbf{y}}^{(1)}=\mathbf{V}e_1$. Since $\mathbf{y}^{(1)}=\mathbf{0}$ and $\|\hat{\mathbf{y}}^{(1)}-\mathbf{y}^{(1)}\|_2\le \mathcal{L}(f)\le \eta$, the claim follows.
\end{proof}

\begin{lemma}\label{lem:softmax-monotone}
Assume the attention mechanism is softmax. Fix any query $\mathbf{q}\in\mathbb{R}^n$ and two candidate sets of keys $S\subseteq T\subset\mathbb{R}^n$. For the softmax probabilities
\begin{align*}
\sigma_S(\mathbf{k}) &= \frac{\exp(\mathbf{q}^\top \mathbf{k})}{\sum_{\mathbf{r}\in S} \exp(\mathbf{q}^\top \mathbf{r})},\\
\sigma_T(\mathbf{k}) &= \frac{\exp(\mathbf{q}^\top \mathbf{k})}{\sum_{\mathbf{r}\in T} \exp(\mathbf{q}^\top \mathbf{r})},
\end{align*}
we have $\sigma_T(\mathbf{k})\le \sigma_S(\mathbf{k})$ for every $\mathbf{k}\in S$.
\end{lemma}

\begin{proof}
The denominators satisfy 
\begin{align*}
\sum_{\mathbf{r}\in T}\exp(\mathbf{q}^\top \mathbf{r})
&=\sum_{\mathbf{r}\in S}\exp(\mathbf{q}^\top \mathbf{r})\\
&\quad+\sum_{\mathbf{r}\in T\setminus S}\exp(\mathbf{q}^\top \mathbf{r})\\
&\ge \sum_{\mathbf{r}\in S}\exp(\mathbf{q}^\top \mathbf{r}),
\end{align*}
while the numerator for a fixed $\mathbf{k}\in S$ is the same in both fractions.
\end{proof}

\begin{lemma}\label{lem:reductions}
Assume the attention mechanism is softmax. Consider any sequence from $\mathcal{D}$ (\cref{sec:task-def}) and any indices $1<i$ and $1<i<h$. Then:
\begin{enumerate}[leftmargin=*,nosep]
\item (Self-reduction) Let $\widetilde{\alpha}_{2,2}$ denote the attention weight on the second token in the length-2 prefix $(\texttt{BOS},\mathbf{x}^{(i)})$, computed with the same $(\mathbf{W}_Q,\mathbf{W}_K)$. Then $\alpha_{i,i}\le \widetilde{\alpha}_{2,2}$.
\item (Pairwise reduction) Let $\widetilde{\alpha}_{3,2}$ denote the attention weight on the second token in the length-3 prefix $(\texttt{BOS},\mathbf{x}^{(i)},\mathbf{x}^{(h)})$, computed with $(\mathbf{W}_Q,\mathbf{W}_K)$. Then $\alpha_{h,i}\le \widetilde{\alpha}_{3,2}$.
\end{enumerate}
\end{lemma}

\begin{proof}
For (1), at real position $i$ the query equals $\mathbf{x}^{(i)}\mathbf{W}_Q$. Let $S$ be the two keys $\{\mathbf{W}_K\mathbf{x}^{(1)},\mathbf{W}_K\mathbf{x}^{(i)}\}$ and $T=\{\mathbf{W}_K\mathbf{x}^{(k)}:k\le i\}$. Lemma~\ref{lem:softmax-monotone} (with this fixed query) gives the claim, noting that $\widetilde{\alpha}_{2,2}=\sigma_S(\mathbf{W}_K\mathbf{x}^{(i)})$ and $\alpha_{i,i}=\sigma_T(\mathbf{W}_K\mathbf{x}^{(i)})$. 

For (2), at real position $h$ the query equals $\mathbf{x}^{(h)}\mathbf{W}_Q$. Let $S=\{\mathbf{W}_K\mathbf{x}^{(1)},\mathbf{W}_K\mathbf{x}^{(i)},\mathbf{W}_K\mathbf{x}^{(h)}\}$ and $T=\{\mathbf{W}_K\mathbf{x}^{(k)}:k\le h\}$; apply Lemma~\ref{lem:softmax-monotone} as before.
\end{proof}

\begin{lemma}\label{lem:self-O-eta}
In the setting of \cref{lem:bos-small}, assume the attention mechanism is softmax. For every sequence in $\mathrm{support}(\mathcal{D})$ and every non-trigger position $1<i \neq j$,
\[
\big\|\alpha_{i,i}\mathbf{V}\mathbf{x}^{(i)}\big\|_2 \le 2\eta .
\]
\end{lemma}

\begin{proof}
Fix $i$ and consider the length-2 prefix $(\texttt{BOS},\mathbf{x}^{(i)})$. At its position $2$ (which is pre-trigger), the output equals
\[
\hat{\mathbf{y}}^{(2)} = \widetilde{\alpha}_{2,1}\mathbf{V}e_1 + \widetilde{\alpha}_{2,2}\mathbf{V}\mathbf{x}^{(i)},
\]
with target $\mathbf{y}^{(2)}=\mathbf{0}$. Hence
\begin{align*}
\big\|\widetilde{\alpha}_{2,2}\mathbf{V}\mathbf{x}^{(i)}\big\|_2
&\le \|\hat{\mathbf{y}}^{(2)}\|_2+\|\widetilde{\alpha}_{2,1}\mathbf{V}e_1\|_2\\
&\le \eta+\eta = 2\eta,
\end{align*}
using Lemma~\ref{lem:bos-small} for the BOS term. By Lemma~\ref{lem:reductions}(1), $\alpha_{i,i}\le \widetilde{\alpha}_{2,2}$, and multiplying both sides by the fixed vector $\mathbf{V}\mathbf{x}^{(i)}$ yields the result.
\end{proof}

\begin{lemma}\label{lem:pairwise-O-eta}
  In the setting of \cref{lem:bos-small}, assume the attention mechanism is softmax. For every sequence in $\mathrm{support}(\mathcal{D})$ and every pair of non-trigger indices $1<i<h$ with $i,h \neq j$:
\[
\big\|\alpha_{h,i}\mathbf{V}\mathbf{x}^{(i)}\big\|_2 \le 4\eta .
\]
\end{lemma}

\begin{proof}
Consider first the length-3 prefix $(\texttt{BOS},\mathbf{x}^{(i)},\mathbf{x}^{(h)})$. At position $3$ (pre-trigger), with target $\mathbf{y}^{(3)}=\mathbf{0}$,
\[
\hat{\mathbf{y}}^{(3)}
= \widetilde{\alpha}_{3,1}\mathbf{V}e_1
+ \widetilde{\alpha}_{3,2}\mathbf{V}\mathbf{x}^{(i)}
+ \widetilde{\alpha}_{3,3}\mathbf{V}\mathbf{x}^{(h)}.
\]
Therefore,
\begin{align*}
\big\|\widetilde{\alpha}_{3,2}\mathbf{V}\mathbf{x}^{(i)}\big\|_2
&\le \|\hat{\mathbf{y}}^{(3)}\|_2+\|\widetilde{\alpha}_{3,1}\mathbf{V}e_1\|_2\\
&\quad+\|\widetilde{\alpha}_{3,3}\mathbf{V}\mathbf{x}^{(h)}\|_2\\
&\le \eta+\eta+2\eta = 4\eta,
\end{align*}
using Lemma~\ref{lem:bos-small} for the BOS term and Lemma~\ref{lem:self-O-eta} for the self term. By Lemma~\ref{lem:reductions}(2), $\alpha_{h,i}\le \widetilde{\alpha}_{3,2}$. Multiplying by $\mathbf{V}\mathbf{x}^{(i)}$ gives the result.
\end{proof}

\begin{lemma}\label{lem:trigger-large}
In the setting of \cref{lem:bos-small}, assume the attention mechanism is softmax. For every sequence in $\mathrm{support}(\mathcal{D})$ with trigger at position $j$,
\[
\|\mathbf{V}\mathbf{x}^{(j)}\|_2 \ge 1 - 2\eta.
\]
\end{lemma}

\begin{proof}
Consider a sequence where the trigger is at position $j=2$. The target output at position 2 is $\mathbf{y}^{(2)} = \mathbf{x}^{(2)}$.
The model output is
\[
\hat{\mathbf{y}}^{(2)} = \alpha_{2,1}\mathbf{V}e_1 + \alpha_{2,2}\mathbf{V}\mathbf{x}^{(2)}.
\]
We know $\|\mathbf{y}^{(2)} - \hat{\mathbf{y}}^{(2)}\|_2 \le \eta$ and $\|\mathbf{V}e_1\|_2 \le \eta$ (\cref{lem:bos-small}).
By triangle inequality, $\|\mathbf{y}^{(2)} - \alpha_{2,2}\mathbf{V}\mathbf{x}^{(2)}\|_2 \le 2\eta$.
Since $(\mathbf{y}^{(2)})_2 = 1$ (trigger indicator), we have $|1 - \alpha_{2,2}(\mathbf{V}\mathbf{x}^{(2)})_2| \le 2\eta$.
Since $\alpha_{2,2} \le 1$, this implies $(\mathbf{V}\mathbf{x}^{(2)})_2 \ge 1 - 2\eta$, so $\|\mathbf{V}\mathbf{x}^{(2)}\|_2 \ge 1 - 2\eta$.
\end{proof}

\begin{lemma}\label{lem:axis-separated}
  Let $n \in \mathbb{N}_{\geq 1}$ and $X=(X_1,\dots,X_n)\sim\mu^{\otimes n}$, where $\mu$ has a density $g$ bounded by $M:=\sup_{x\in\mathbb{R}} g(x) < \infty$. 
  Fix $\delta\in(0,1]$. Then there exists some $\varepsilon' \in \mathbb{R}_{>0}$ such that if a measurable set $E\subset\mathbb{R}^n$ satisfies $\mathbb{P}(X\in E)\ge \delta$, then for every coordinate $j\in\{1,\dots,n\}$ there exist $x,y\in E$ such that
  \[
  x_k=y_k\; \text{for all }k\neq j,
  \quad\text{and}\quad 
  |x_j-y_j|\;\ge\;\varepsilon',
  \]
  \end{lemma}
  
  \begin{proof}
  Fix $j$ and, for $z\in\mathbb{R}^{n-1}$, set $E_j(z):=\{t\in\mathbb{R}:(z_1,\dots,z_{j-1},t,z_{j+1},\dots)\in E\}$. By Fubini and independence,
  \[
  \mathbb{P}(X\in E)\;=\;\int \mu\!\left(E_j(z)\right)\,d\mu^{\otimes (n-1)}(z).
  \]
  Since $\mu$ has density $g$ bounded by $M$, for any measurable $A\subset\mathbb{R}$ we have $\mu(A)\le M\,\lambda(A)$, where $\lambda$ is the Lebesgue measure. Hence
  \begin{align*}
  \delta &\;\le\; \int \mu(E_j(z))\,d\mu^{\otimes (n-1)}(z)\\
  &\;\le\; M\int \lambda(E_j(z))\,d\mu^{\otimes (n-1)}(z).
  \end{align*}
  Therefore there exists $z$ with $\lambda(E_j(z))\ge \delta/M$. Any set $A\subset\mathbb{R}$ with Lebesgue measure $\lambda(A)$ has diameter at least $\lambda(A)-\eta$ for any $\eta \in \mathbb{R}_{>0}$, so we can choose $t_1,t_2\in E_j(z)$ with $|t_1-t_2|\ge \delta/M - \eta$ with $\eta<\delta/2M$. Setting $\varepsilon'=\delta/2M$ and taking $x,y$ to match $z$ on all coordinates $k\neq j$ and have $j$-th coordinates $t_1,t_2$ respectively gives the claim.
  \end{proof}

\begin{lemma}\label{lem:multilayer-unroll}
Let $f=f^{(D)}\circ\cdots\circ f^{(1)}$ be a $D$-layer causal softmax self-attention model as in \S\ref{sec:attention-weights} .
For each layer $d\in\{1,\ldots,D\}$ write
\[
\mathbf{V}^{(d)} \;:=\; \mathbf{W}_O^{(d)}\mathbf{W}_V^{(d)}.
\]
\[
  \mathbf{V} \;:=\; \mathbf{V}^{(D)}\mathbf{V}^{(D-1)}\cdots \mathbf{V}^{(1)}
\]
Then for every input sequence $\mathbf{x}$ and every position $i\in[L]$, there exist coefficients
$\beta_{i,1}(\mathbf{x}),\ldots,\beta_{i,i}(\mathbf{x})$ such that
\begin{equation}\label{eq:multilayer-unroll}
f(\mathbf{x})^{(i)} \;=\; \sum_{k=1}^{i} \beta_{i,k}(\mathbf{x})\,\mathbf{V}\mathbf{x}^{(k)}.
\end{equation}
Moreover, for each $i$ we have $\beta_{i,k}(\mathbf{x})\ge 0$ for all $k\le i$ and
\[
\sum_{k=1}^{i}\beta_{i,k}(\mathbf{x}) \;=\; 1.
\]
\end{lemma}

\begin{proof}
Let $\mathbf{z}^{(0)}:=\mathbf{x}$ and for $d\ge 1$ let $\mathbf{z}^{(d)}:=f^{(d)}(\mathbf{z}^{(d-1)})$.
Write $\alpha^{(d)}_{i,k}$ for the (softmax) attention weight in layer $d$ from position $i$ to key $k\le i$.
By definition of a single layer,
\[
\mathbf{z}^{(d)}{}^{(i)} \;=\; \sum_{k\le i}\alpha^{(d)}_{i,k}\,\mathbf{V}^{(d)}\mathbf{z}^{(d-1)}{}^{(k)}.
\]
Define $\beta^{(1)}_{i,k}:=\alpha^{(1)}_{i,k}$, and for $d\ge 2$ define recursively
\[
\beta^{(d)}_{i,k}
\;:=\;
\sum_{\ell:\,k\le \ell\le i}\alpha^{(d)}_{i,\ell}\,\beta^{(d-1)}_{\ell,k}.
\]
A direct induction on $d$ gives
\[
\mathbf{z}^{(d)}{}^{(i)}
\;=\;
\sum_{k\le i}\beta^{(d)}_{i,k}\,\mathbf{V}^{(d)}\cdots \mathbf{V}^{(1)}\mathbf{x}^{(k)}.
\]
Nonnegativity and the row-sum identity follow since each $\alpha^{(d)}_{i,\cdot}$ is a probability vector.
Taking $d=D$ and setting $\beta_{i,k}:=\beta^{(D)}_{i,k}$ yields \eqref{eq:multilayer-unroll}.
\end{proof}

\begin{lemma}\label{lem:beta22-product}
In the setting of Lemma~\ref{lem:multilayer-unroll}, for any input sequence $\mathbf{x}$ we have
\[
\beta_{2,2}(\mathbf{x})
\;=\;
\prod_{d=1}^{D}\alpha^{(d)}_{2,2}(\mathbf{x}),
\]
where $\alpha^{(d)}_{2,2}(\mathbf{x})$ is the attention weight at position $2$ attending to position $2$ in layer $d$.
\end{lemma}

\begin{proof}
In the recursion from the proof of Lemma~\ref{lem:multilayer-unroll}, note that position $1$ is causal and thus never depends on token $2$, directly yielding the product formula.
\end{proof}

\begin{lemma}\label{lem:bos-small-multilayer}
In the setting of Lemma~\ref{lem:multilayer-unroll}, if the loss $\mathcal{L}(f)$ (see \cref{sec:task-def}) satisfies $\mathcal{L}(f)\le \eta$ then
\[
\|\mathbf{V}e_1\|_2 \;\le\; \eta.
\]
\end{lemma}

\begin{proof}
By causality, at position $i=1$ every layer attends only to position $1$, hence $f(\mathbf{x})^{(1)}=\mathbf{V}\mathbf{x}^{(1)}=\mathbf{V}e_1$.
Since $\mathbf{y}^{(1)}=\mathbf{0}$ and $\|f(\mathbf{x})^{(1)}-\mathbf{y}^{(1)}\|_2\le \mathcal{L}(f)\le \eta$, the claim follows.
\end{proof}

\begin{lemma}\label{lem:beta22-O-eta}
In the setting of Lemma~\ref{lem:multilayer-unroll}, assume softmax attention and that the loss $\mathcal{L}(f)$ (see \cref{sec:task-def}) satisfies $\mathcal{L}(f)\le \eta$.
Then for every $\mathbf{x}$ in $\mathrm{support}(\mathcal{D})$ with trigger position $j\ge 3$ we have that
\[
\big\|\beta_{2,2}(\mathbf{x})\,\mathbf{V}\mathbf{x}^{(2)}\big\|_2 \;\le\; 2\eta.
\]
\end{lemma}

\begin{proof}
Since $j\ge 3$, position $2$ is pre-trigger and the target satisfies $\mathbf{y}^{(2)}=\mathbf{0}$.
By Lemma~\ref{lem:multilayer-unroll} with $i=2$,
\[
f(\mathbf{x})^{(2)}
=
\beta_{2,1}(\mathbf{x})\,\mathbf{V}e_1
+
\beta_{2,2}(\mathbf{x})\,\mathbf{V}\mathbf{x}^{(2)}.
\]
Thus
\begin{align*}
\big\|\beta_{2,2}(\mathbf{x})\,\mathbf{V}\mathbf{x}^{(2)}\big\|_2
&\le
\|f(\mathbf{x})^{(2)}\|_2 \\
&\quad+
\beta_{2,1}(\mathbf{x})\,\|\mathbf{V}e_1\|_2 \\
&\le
\eta+\eta \\
&= 2\eta,
\end{align*}
using $\|f(\mathbf{x})^{(2)}-\mathbf{y}^{(2)}\|_2\le \eta$, 
$\beta_{2,1}(\mathbf{x})\le 1$, and Lemma~\ref{lem:bos-small-multilayer}. 
\end{proof}

\section{Related Work}\label{sec:related_extended}

\paragraph{Theory and analyses of attention sinks.}
Several recent works study attention sinks directly, aiming to characterize why they arise, what they correlate with, and how they are implemented. \citet{ranmilo2026mechanisticaccountattentionsinks} identify one concrete sink-forming circuit in GPT-2 and show that sinks can persist when its specific components are absent, suggesting that sink behavior may be realized through different circuits across architectures; our results help explain this, since if sinks are functionally necessary, optimization may naturally find different architecture-specific ways to implement them. \citet{barbero2025llmsattendtoken} argue (theoretically and empirically) that first-token sinks can act as a stabilizing mechanism against over-mixing, and analyze how factors like depth, context length, and packing influence sink strength.”
\citet{cancedda2024spectralfiltersdarksignals} connect sink behavior to spectral structure in the vocabulary embedding/unembedding operators, attributing sinking to ``dark'' (tail-spectrum) components.
\citet{ruscio2025sinkinggeometricapproachattention} view sinks as learned ``reference-frame anchors'' in representation space and show that the resulting anchoring pattern depends strongly on architectural choices, especially the positional encoding.
\citet{queipodellano2026attentionsinkscompressionvalleys} connect attention sinks to ``compression valleys'' (layers where token representations become unusually low-entropy/compressed), showing both tend to emerge when the BOS token develops extremely large residual-stream activations. 
\citet{qiu2026unifiedsinks} study attention sinks together with ``residual sinks'' (persistent large activations in a few residual-stream dimensions) and argue these outliers interact with normalization (softmax/RMSNorm) to rescale the remaining components, supporting stable training.
\citet{sok2026garbageattention} treat strong BOS-focused heads—especially in later layers—as a marker of functional redundancy and propose a pruning criterion based on sink scores.
\citet{hong2025variance} attribute softmax-driven attention entropy collapse (attention concentrating onto a single token) to variance sensitivity of the logits and propose entropy-stable alternatives.
\citet{zhang2025catchtagrelease} link sink tokens to large-norm outlier directions in LLM representations and RoPE-focused analyses similarly tie sink behavior to structured frequency artifacts and Q/K ``massive values'' \citep{jin2025massivevalues,xiong2025dope}. These ``massive values'' were recently revisited in \citet{sun2026spike_sparse_sink}, which argues that massive activations and attention sinks are largely decoupled: spikes can be suppressed via normalization changes while sinks persist.
We complement these with a different angle: rather than studying how sinks emerge during training, we ask whether they are structurally \emph{necessary} for certain computations.
We prove that any softmax attention model solving a natural trigger-conditional task must develop a sink, regardless of the training procedure or optimization dynamics (\cref{thm:main,thm:multilayer}).

\paragraph{Softmax normalization implications.}
In standard attention, the softmax turns scores into nonnegative weights that sum to one.
\citet{richter2020normalizedattentionprobabilitycage} analyze how this simplex constraint can restrict attention behavior and discuss alternatives that relax or replace softmax normalization.
\citet{velickovic2024softmaxnotenough} prove that softmax-based mechanisms can fail to maintain increasingly sharp selection as the problem size grows, leading to degraded behavior under distribution shift when near-argmax behavior is required.
We provide a concrete natural task where this constraint is provably the cause of sink formation: a model that must aggregate context on a trigger token and output zero otherwise cannot avoid a sink under softmax normalization (\cref{thm:main}), whereas ReLU attention—which lacks the simplex constraint—solves the same task without any sink (\cref{thm:relu}).

\paragraph{Mitigating sinks.}
Alongside analyses, multiple papers propose sink-targeted interventions.
This includes modified attention normalizations explicitly designed to avoid sinks \citep{zuhri2025softpick,huang2026tda}, as well as training procedures tailored to long-context regimes, including sliding-window attention that explicitly addresses attention-sink issue\citep{fu2025swat}.
For inference-time efficiency, \citet{su2025kvsink,hosseini2026innerq} analyze how KV-cache quantization can disrupt sink behavior and propose predicting and preserving sink tokens during quantization.
Mitigation has also been studied for closely related collapse modes of attention: \citet{hong2025variance} analyze softmax-driven entropy collapse (attention concentrating onto a single token) and propose alternatives aimed at stabilizing attention entropy, while \citet{hankemeier2026diagonalsink} study diagonal/temporal self-attention sinks and introduce regularizers to counter them.
In a different setting, \citet{lookbothways2025nosink} show that attention sinks degrade training-free conversion of decoder-only LLMs into text encoders, and reduce this effect by enabling bidirectional attention and masking the first token in attention.
In multimodal and AV settings, sink patterns have similarly motivated mitigation strategies aimed at reducing hallucination and stabilizing activations \citep{zhang2024eah,anand2025avsr_sinks}. \citet{lu2025artifactsattentionsinks} analyze attention sinks as a structured artifact in Vision Transformers and leverage this structure to derive efficient approximation schemes. Moreover, in these settings, sinks have been explicitly regularized in the context of harmful fine-tuning \citep{liu2026surgery}. Sinks have also been studied in alignment and security contexts where \citet{shang2025forgetting} leverage sink behavior as a pathway for backdooring unlearning procedures. Finally, circuit-level interventions have also been explored in regimes where sink-related circuitry correlates with repeated-token failures \citep{yona2025repeatedtokens}.
Our necessity results offer a principled lens for evaluating such interventions: for trigger-conditional circuits, the sink is the mechanism enabling the computation, so strategies that operate \emph{within} softmax (penalizing BOS attention, spreading mass, post-hoc reweighting) risk degrading the circuit without addressing the root cause.
The contrast with ReLU attention (\cref{thm:relu,sec:implications}) suggests that relaxing the normalization constraint is the more fundamental direction.

\paragraph{Usefulness of sinks.}
Other work treats sinks as a useful computational primitive rather than an artifact to eliminate.
Our work formalizes this intuition: for trigger-conditional behaviors—where a model must aggregate context on a trigger while outputting zero elsewhere—the sink is not merely a convenient implementation choice but a \emph{provably necessary} consequence of softmax normalization (\cref{thm:main,thm:multilayer}).
\citet{zhang2025catchtagrelease} link sink tokens to representation outliers and argue that simple structural conditions (e.g., low-rank attention structure) can be sufficient to induce sinks that support concrete computations such as averaging and retrieval---a viewpoint that is closely aligned with our trigger-conditional setting.
Sinks have been argued to induce or support attention-layer specialization, including MoE-like effects within attention \citep{fu2026sinkmoe}.
\citet{sandovalsegura2025dormant} use sink dominance to identify ``dormant'' heads and validate their redundancy via head ablations.
In addition, BOS-sink heads have been treated as a locus of redundancy that can be targeted for model simplification via sink-aware pruning \citep{sok2026garbageattention}.
In large vision-language models \citep{luo2025sinksinkvisualinformation} show that high-norm ViT sink tokens encode high-level semantic concepts and serve as important visual information pathways into the LLM, and propose methods to better leverage them.
Related ideas appear in diffusion LMs as well, where introducing an explicit sink token is used to stabilize sink behavior across steps \citep{zhang2026onetoken} and where sink locations can be transient across denoising steps, motivating sink-aware pruning that targets unstable sinks \citep{myrzakhan2026sinkawarepruning}.

\begin{figure*}[!t]
  \centering
  \includegraphics[width=\textwidth]{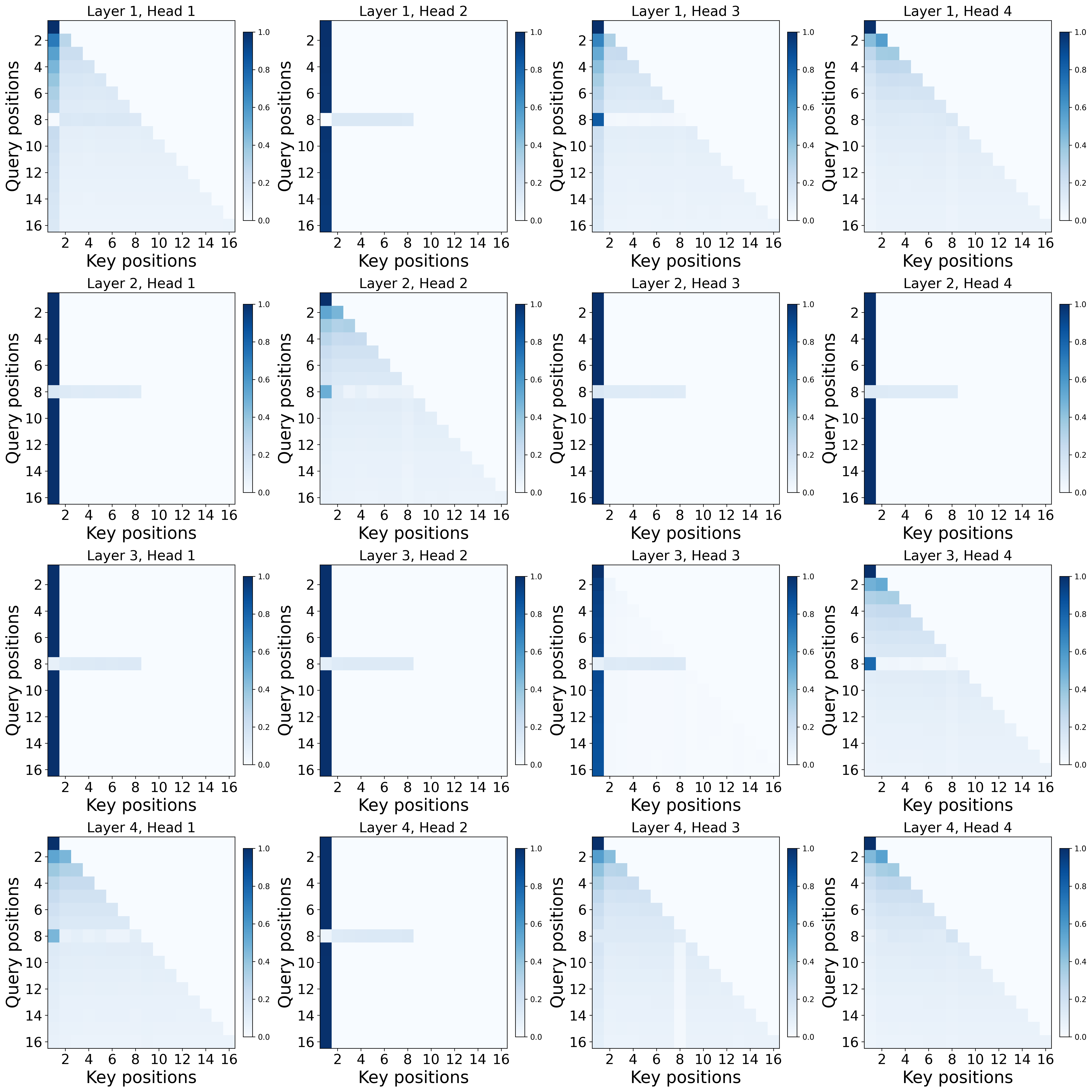}
  \caption{\textbf{Softmax attention: 4-layer 4-head model.} Representative attention patterns on a single test input showing strong sink at least in one head across all layers.}
  \label{fig:appendix-softmax-4h4d}
\end{figure*}

\begin{figure*}[!t]
  \centering
  \includegraphics[width=\textwidth]{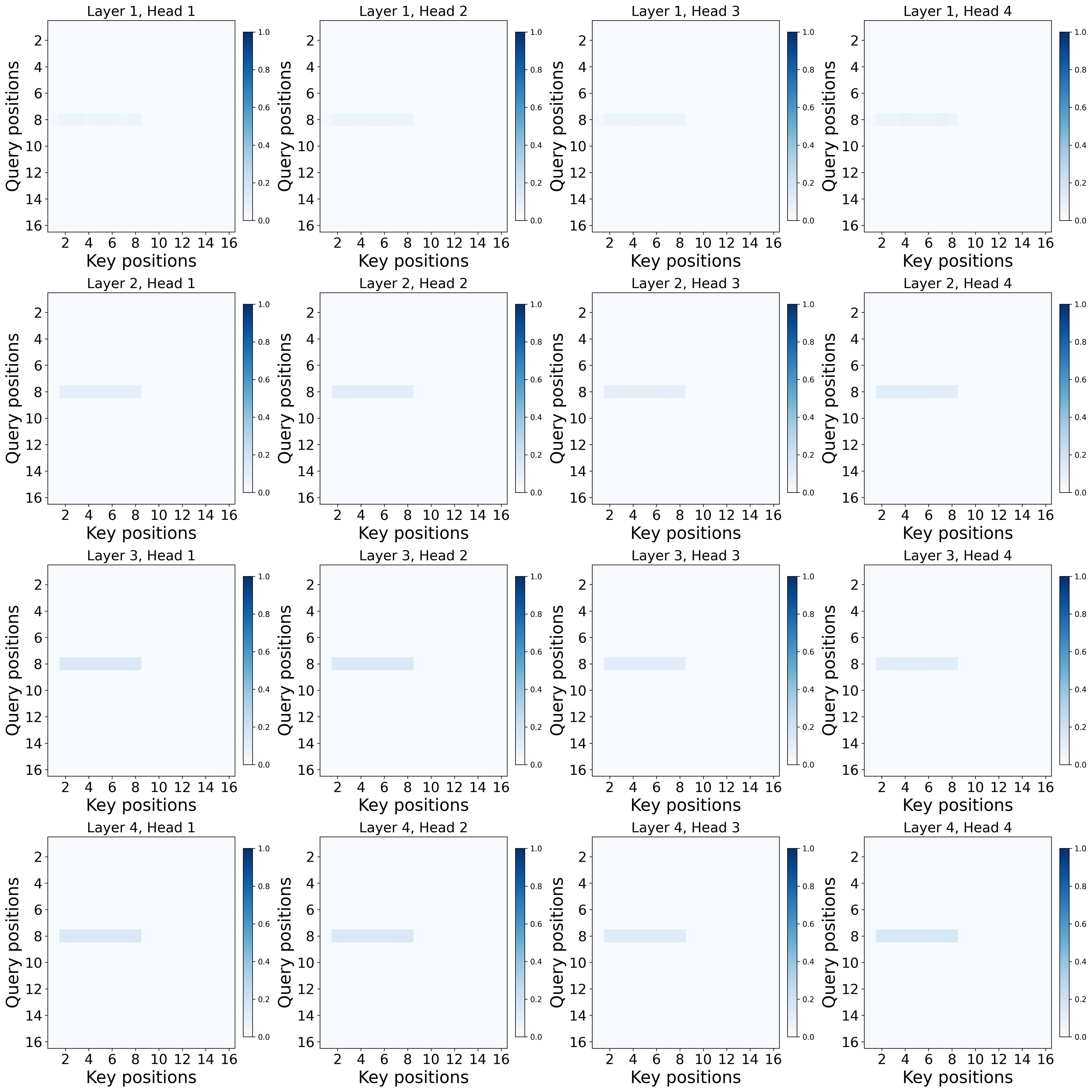}
  \caption{\textbf{ReLU attention: 4-layer 4-head model.} Representative attention patterns on a single test input showing absence of sink behavior across all layers.}
  \label{fig:appendix-relu-4h4d}
\end{figure*}

\end{document}